\documentclass[lettersize,journal]{IEEEtran}
\usepackage{amsmath,amsfonts}
\usepackage{algorithm}
\usepackage{algcompatible}
\algnewcommand\algorithmicreturn{\textbf{return}}
\algnewcommand\RETURN{\State \algorithmicreturn}%
\usepackage{array}
\usepackage[caption=false,font=normalsize,labelfont=sf,textfont=sf]{subfig}
\usepackage{textcomp}
\usepackage{stfloats}
\usepackage{url}
\usepackage{verbatim}
\usepackage{graphicx}
\usepackage{cite}
\usepackage{enumitem}
\usepackage{booktabs}
\usepackage{multirow}
\usepackage{amsthm}
\usepackage{hyperref}
\usepackage{url}
\newlist{steps}{enumerate}{1}
\setlist[steps, 1]{leftmargin=1.2cm, label = Step \arabic*:}

\newtheorem{theorem}{Theorem}
\newtheorem{lemma}{Lemma} 

\newtheorem{definition}{Definition}
\newtheorem{corollary}{Corollary}
\theoremstyle{remark}
\newtheorem*{remark}{Remark}
\newtheorem{hyp}{Hypothesis}
\usepackage{threeparttable}
\begin{document}

\title{Time-Optimal Planning for Long-Range Quadrotor Flights: An Automatic Optimal Synthesis Approach}

\author{Chao Qin$^{1}$, Jingxiang Chen$^{2}$, Yifan Lin$^{2}$, Abhishek Goudar$^{1}$, Angela P. Schoellig$^{1,3}$ and Hugh H.-T. Liu$^{1}$
\thanks{$^{1}$ The authors are with the University of Toronto Institute for Aerospace Studies, Canada. They are also associated with the University of Toronto Robotics Institute.}
\thanks{$^{2}$ The authors are with the University of Toronto Division of Engineering Science, Canada.}
\thanks{$^{3}$ The authors are with the Technical University of Munich, Germany. They are also associated with the Munich Institute of Robotics and Machine Intelligence (MIRMI).}
}

\markboth{Journal of \LaTeX\ Class Files,~Vol.~14, No.~8, August~2021}%
{Shell \MakeLowercase{\textit{et al.}}: A Sample Article Using IEEEtran.cls for IEEE Journals}


\maketitle

\begin{abstract}
Time-critical tasks such as drone racing typically cover large operation areas. However, it is difficult and computationally intensive for current time-optimal motion planners to accommodate long flight distances since a large yet unknown number of knot points is required to represent the trajectory. We present a polynomial-based automatic optimal synthesis (AOS) approach that can address this challenge. Our method not only achieves superior time optimality but also maintains a consistently low computational cost across different ranges while considering the full quadrotor dynamics. First, we analyze the properties of time-optimal quadrotor maneuvers to determine the minimal number of polynomial pieces required to capture the dominant structure of time-optimal trajectories. This enables us to represent substantially long minimum-time trajectories with a minimal set of variables. Then, a robust optimization scheme is developed to handle arbitrary start and end conditions as well as intermediate waypoints. Extensive comparisons show that our approach is faster than the state-of-the-art approach by orders of magnitude with comparable time optimality. Real-world experiments further validate the quality of the resulting trajectories, demonstrating aggressive time-optimal maneuvers with a peak velocity of 8.86 m/s. 
\end{abstract}



\section*{SUPPLEMENTARY MATERIAL}
\noindent\textbf{Video:} \url{https://www.youtube.com/watch?v=WLtnzOf9400}\\
\textbf{Code:} \url{https://github.com/FSC-Lab/aos_time_optimal}

\section{Introduction}\label{sec:intro}
\IEEEPARstart{Q}UADROTORS are one of the most agile and mechanically simple flying robots \cite{tal2020accurate} that have been widely used in time-critical tasks such as delivery, search \& rescue, and drone race \cite{foehn2022alphapilot}. These applications often cover a large and complicated operation area (e.g., using a stadium as a race course \cite{barin2017understanding}), and thus require time-optimal planners to scale well with long flight distances as well as a large number of waypoints, without diminishing the solution quality in terms of time optimality and dynamical feasibility. Unfortunately, state-of-the-art approaches cannot meet all these requirements for long-range tasks, which limits the potential of autonomous aerial vehicles in these time-sensitive scenarios.

\begin{figure}[!t]
\centering
\includegraphics[width=0.48\textwidth]{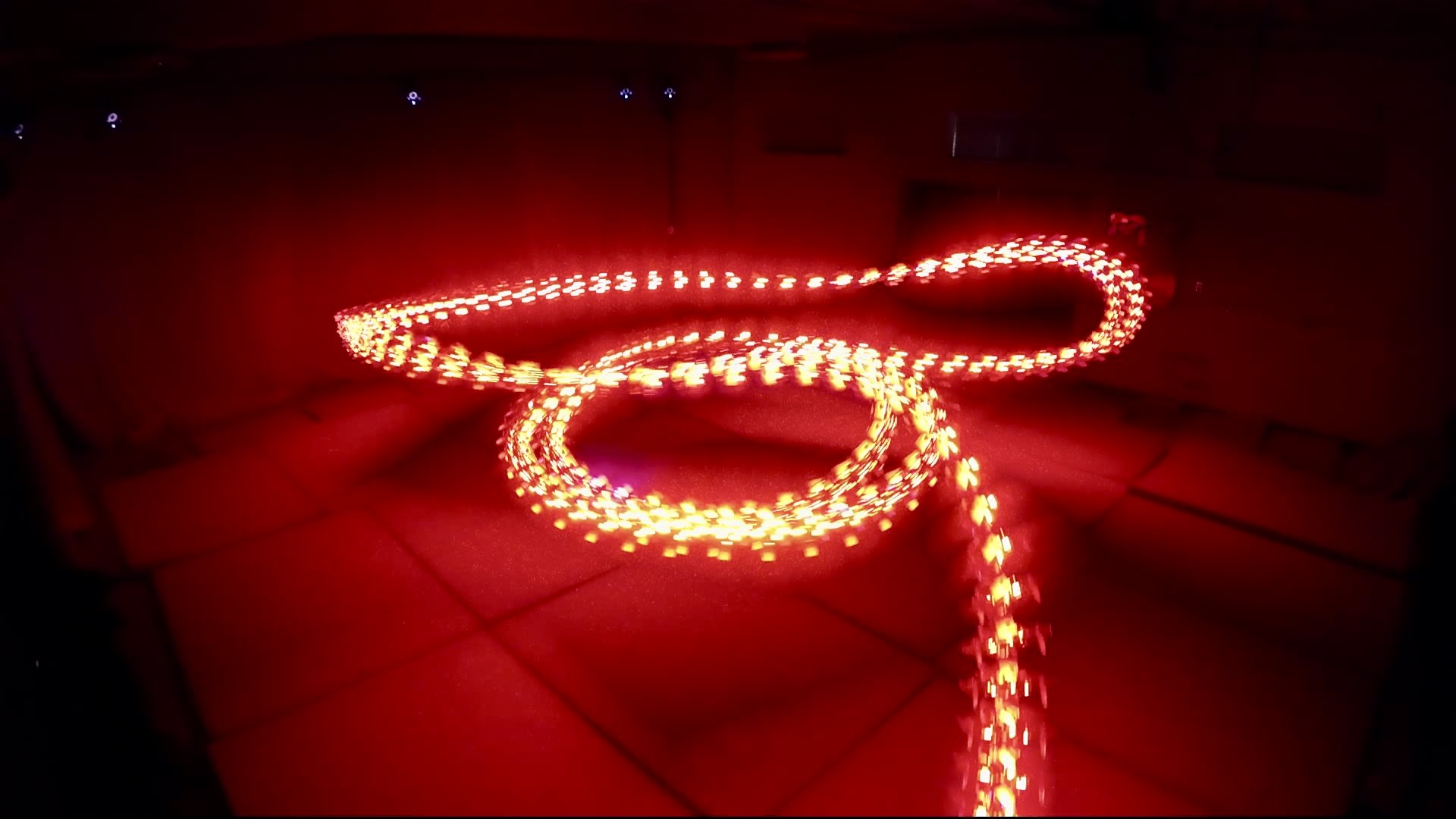}
\caption{Time-optimal trajectory generated by the proposed AOS approach to pass 19 waypoints and executed by a quadrotor platform in a motion capture room. Our approach achieves near-optimal racing performance and can handle long-range time-critical tasks with superior efficiency. 
}\label{fig_intro}
\end{figure}

State-of-the-art approaches \cite{spedicato2017minimum, foehn2021time, zhou2023efficient, romero2022model, westenberger2022efficient, toumieh2022near} employ time-discretized trajectory to construct the minimum-time path planning problem. This strategy often leads to optimal timings when the endpoint is relatively close to the initial point. However, it becomes problematic if a distant destination or waypoint is involved. This is because as the target distance increases, the optimal duration must increase correspondingly, necessitating more knot points (or nodes) \cite{tang2024direct} to maintain a short sampling period; otherwise, the integration error may manifest and thereby undermine the trajectory quality. As a result, a soaring problem complexity may arise, resulting in minutes or even hours of time consumption \cite{romero2022time}. Furthermore, the presence of a free end-time poses a challenge in determining the number of samples for trajectory discretization, and a considerable tuning effort is required to find a reasonable value for every waypoint layout. Although polynomial-based methods \cite{mellinger2011minimum, richter2016polynomial, wang2022geometrically, han2021fast, de2023alternating, ryou2021multi, qin2023time, wang2023polynomial} can circumvent the aforementioned problems, they yield significantly longer trajectory durations because smooth polynomials struggle to represent rapid state or input changes \cite{foehn2021time}. Moreover, it remains an open question under what conditions a polynomial-based method can achieve or approach true time optimality. 

In this paper, we show that the piecewise polynomial representation can yield trajectories that are nearly time-optimal, provided that the assigned piece number is large enough to account for all substantial changes in the collective thrust and rotational rate. We present a time-optimal planner for quadrotors that scales well with mission ranges in two-state and multi-waypoint flight problems as shown in Fig. \ref{fig_intro}. By leveraging the structure of time-optimal control (TOC) and the differential-flat properties of the quadrotor system \cite{mellinger2011minimum}, we demonstrate an efficient approach for solving time-optimal planning problems with arbitrary start and end points as well as waypoint layouts. Moreover, it reaches exceptional trajectory quality by accounting for the full quadrotor dynamics as well as actuation limits on single-rotor thrusts. We evaluate the proposed algorithm in a variety of scenarios and confirm its superior solution quality. For instance, on the Split-S race track \cite{song2023reaching}, our approach is merely 1.7\% slower than the state-of-the-art in terms of timing, with an order-of-magnitude faster computation speed.

The key to achieving true time optimality hinges on a thorough understanding of time-optimal maneuvers. We observe that for many systems, time-optimal trajectories can be represented as a concatenation of multiple trajectory segments \cite{soueres1996shortest}. Moreover, their structures can be classified into several categories, and the optimal trajectories follow the same pattern within each category regardless of the distance between the start and end points. Take a double integrator as an example. It is well-known that its optimal control policy is bang-bang with at most one switch \cite{liberzon2011calculus}. In other words, the resulting minimum-time trajectories must consist of at most two segments for any task. This paper generalizes this principle to quadrotors. As illustrated in Fig. \ref{fig_aos_diagram}, we study the structure of time-optimal quadrotor maneuvers via the necessary condition of optimality from the \textit{Pontryagin Maximum Principle} (PMP). This information sheds light on the minimal number of polynomials required to construct the time-optimal trajectories in the worst-case scenarios. We then formulate an optimization problem to put all pieces together under a time-minimization objective. We term this process \textbf{A}utomatic \textbf{O}ptimal \textbf{S}ynthesis (AOS) as the key idea is still optimal synthesis (OS) \cite{boscain2003optimal} but replacing the manual synthesis process with a numerical optimization. Our contributions can be summarized below:

\begin{itemize}[]
\item We present a time-optimal planner that can tackle long-range quadrotor flights with high solution quality and efficiency.
\item We prove that in minimum-time fixed-endpoint problems of quadrotors, the collective thrust is \textit{bang-bang with at most 5 switches} and the rotational rate is \textit{bang-singular with at most 2 singular arcs and 3 isolated bang arcs}.
\item To our best knowledge, this is the first polynomial-based method capable of yielding solutions comparable to discretization-based methods. And we conducted extensive experiments to validate its real-world performance.
\end{itemize}

\begin{figure*}[!htbp]
\centering
\includegraphics[width=0.9\textwidth]{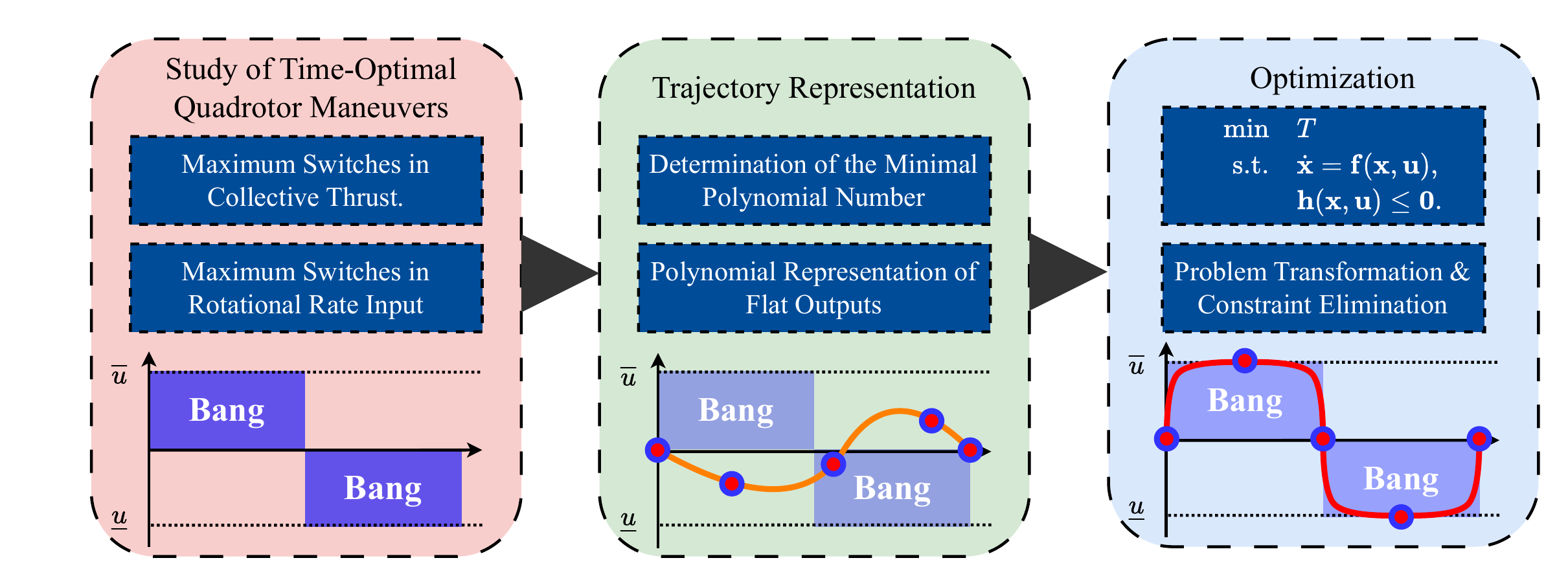}
\caption{Illustration of the AOS approach. \textbf{Step 1)}: understand the bang-bang or bang-singular structure of the optimal collective thrust and rotational rate trajectories. \textbf{Step 2)}: leverage the obtained information to determine an appropriate number of polynomials for trajectory representation. \textbf{Step 3)}: solve the time-optimal planning problem parameterized by the selected polynomials. It is expected that the polynomial can accommodate all rapid state or input changes with reasonable order and thereby generate close-to-optimal results.} \label{fig_aos_diagram}
\end{figure*}

\section{Related Work}\label{sec:related}
\subsection{Direct Methods vs. Indirect Methods}
Existing methods can be categorized into direct methods (DMs) and indirect methods (IMs). DMs perform state-space discretization, followed by formulating a constrained nonlinear programming problem (NLP). Many well-known algorithms fall into this family such as multiple/single shooting methods \cite{bock1984multiple} and direct collocation \cite{kelly2017introduction}. Note that their control inputs are typically parameterized as piecewise constant, which works quite well in TOC problems. However, DMs quickly become intractable as the prediction horizon increases. Moreover, when a minimum sampling time is specified, determining an appropriate number of knot points becomes difficult due to the unknown trajectory duration. To cope with this issue, prior works resort to either the user's experience \cite{foehn2021time} or a constant-velocity model \cite{zhou2023efficient}. However, these strategies are not reliable and often result in optimization failure when dealing with large waypoint layouts.

Different from the DMs' ``first discretize, then optimize '' strategy \cite{diehl2011numerical}, IMs first study the first-order necessary conditions, narrowing down the solution to a finite set of trajectory families. In the second step, a boundary value problem (BVP) can be constructed \cite{osborne1969shooting} to solve the optimal quadrotor maneuvers, as depicted by Hehn et al. \cite{hehn2012performance}. However, the BVP is highly sensitive to the initial guess of the adjoint state, which renders it unstable. An alternative second step is to establish an optimal synthesis---finitely concatenating candidate trajectories until the terminal state is reached---by exploring candidate trajectory families derived from the optimality conditions. One of its greatest achievements is solving Dubin's problem \cite{dubins1957curves}, in which Boissonnat et al. \cite{boissonnat1994shortest} deduced 46 trajectory families in the solution set and Soueres et al. \cite{soueres1996shortest} successfully identified the analytical solutions. However, finding OS for nonlinear systems in 3-D space is a daunting task. To the best of our knowledge, the OS for the two-dimensional quadrotor model has not been discovered, let alone its three-dimensional counterpart. Nevertheless, using a simplified model, Hehn et al. \cite{hehn2015real} provided certain analytical results but with compromised time optimality. Additionally, IMs inherently lack flexibility in incorporating state constraints, rendering them unsuitable for addressing complex tasks involving geometric constraints. Lastly, it is worth pointing out that our approach combines the strengths of IMs and DMs to achieve efficient and accurate solutions within a flexible framework. 

\subsection{Discrete Formulation vs Polynomial Representation}

There are two main options for trajectory representation in time-optimal planning: discretization-based and continuous-time polynomials. It is worth noting that trajectory discretization can be performed either in a temporal coordinate \cite{lai2006time, foehn2021time, zhou2023efficient, shen2023aggressive, romero2022model} or along a reference path \cite{spedicato2017minimum, arrizabalaga2022towards, van2013time}, and all these approaches are deemed discretization-based approaches. Multiple shooting (MS) methods are popular among those utilizing time-discretized trajectories. Particularly, to guarantee uniform sampling times, Foehn et al. \cite{foehn2021time} introduced complementary progress constraints (CPC) to tackle time-optimal waypoint flights. Zhou et al. \cite{zhou2023efficient} reduced runtime by manually specifying knot point numbers for trajectories between two waypoints. Spedicato et al. \cite{spedicato2017minimum} projected the system dynamics onto a reference path so that the total trajectory duration can be considered as an independent decision variable. To mitigate nonconvexity, Arrizabalaga et al. \cite{arrizabalaga2022towards} replaced the minimum-time objective with a quadratic progress-maximization term along the reference path. Mao et al.  \cite{mao2023toppquad} incorporated a more realistic quadrotor state and input bounds such as motor speeds. Although using a spatial coordinate decouples time from the dynamics equations, it implicitly forces the drone to follow the pre-defined path, which is suboptimal in most scenarios. Sampling-based methods \cite{allen2016real, liu2017search, liu2018search} have also been proven effective. Liu et al. \cite{liu2017search} generated a group of short-duration motion primitives and searched for the optimal one. The authors of \cite{foehn2022alphapilot, romero2022time} utilized the bang-bang structure of the point-mass model and sampled the velocity at each waypoint to generate near-optimal solutions. To enhance model accuracy without causing convergence issues, Penicka et al. \cite{penicka2022minimum} employed a hierarchical sampling scheme in which the model is refined at each level. However, it is computationally intensive and prone to error when the space is large.

Polynomial-based approaches are widely used in quadrotor navigation in clustered environments \cite{ji2022elastic, zhou2021raptor, ren2023online, gao2018online}. They typically model the quadrotor's flat outputs to eliminate the need for simulating the high-dimensional state space via numerical integration, which substantially alleviates the computation complexity. Van et al. \cite{van2013time} parameterized a polynomial trajectory expressed in the spatial coordinates, deriving a fixed end-time optimal control problem with path constraints that can be solvable, but it is restricted to a two-dimensional task. Ryou et al. \cite{ryou2021multi} employed Bayesian optimization to learn the optimal time allocation over each polynomial segment from multi-fidelity data. Mellinger et al. \cite{mellinger2011minimum} proposed the well-known minimum-snap trajectory and included time as a term to be minimized. Mueller et al. \cite{mueller2013computationally} computed minimum-time trajectories via two steps; they derived an analytic expression for a polynomial family and created a group of motion primitives with different durations, followed by searching for the minimum-time one. de Vries et al. \cite{de2023alternating} improved the second step by introducing a peak optimization method. Wang et al. \cite{wang2022geometrically} extended the first step from a single polynomial to a piecewise polynomial and provided the derivation process. Their framework enables spatial-temporal deformation of the decision variables, contributing to a highly efficient implementation of smooth trajectory planning subject to geometric constraints. This approach has been applied to time-critical tasks such as drone racing \cite{han2021fast, wang2023polynomial, qin2023time}. Han et al. \cite{han2021fast} focused on the collision avoidance aspect, Wang et al. \cite{wang2023polynomial} concentrated on the dynamic gates, and Qin et al. \cite{qin2023time} investigated the impact of gate shapes on time optimality. However, these works suffer from the inherent smoothness of polynomials and only yield time-suboptimal trajectories.

\section{Properties of Time-Optimal Maneuvers}\label{sec:properties}

This section investigates the bang-singular structure of optimal control policies in the time-optimal two-state problems. To make the analysis tractable, we follow the convention used in \cite{hehn2012performance} and consider a two-dimensional quadrotor model with collective thrusts and rotational rates as the control inputs. 

\subsection{An Account of The Previous Works}
Let's first introduce the cornerstone results in previous works. 

\subsubsection{Non-Dimensional Quadrotor Model}

To facilitate the analysis, Hehn et al. \cite{hehn2012performance} converted the original system to a non-dimensional model as shown in Fig. \ref{fig_quad2d}. Note that any properties derived in this model will equally apply to the original system. To make the paper self-contained, the derivation of the non-dimensional model is described below.

\begin{figure}[!htbp]
\centering
\includegraphics[width=0.35\textwidth]{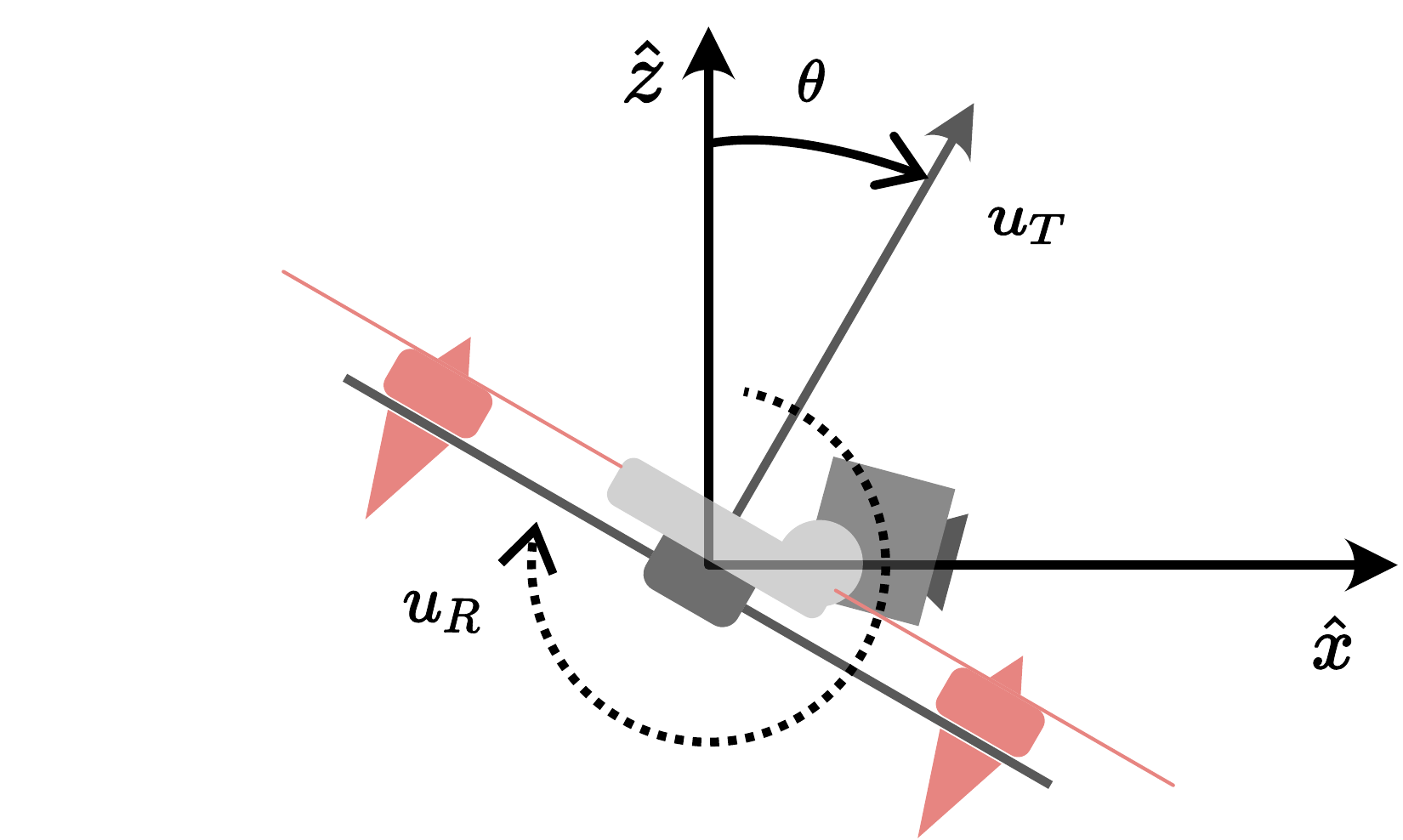}
\caption{Illustration of the non-dimensional quadrotor model to facilitate the analysis of time-optimal maneuvers.}
\label{fig_quad2d}
\end{figure}

Consider a quadrotor model with three degrees of freedom: the horizontal position $x$, the vertical position $z$, and the pitch angle $\theta$ \footnote{Here $\theta$ is 
treated as a value in $\mathbb{R}$ instead of $S$}. There are two inputs to this system, the collective thrust, $F_{T}$, and the pitch rate, $\omega$, which are subject to the following constraints:
\begin{align}
0<\underline{F_{T}}\leq F_{T}\leq\overline{F_{T}},\;
|\omega|\leq\overline{\omega},
\end{align}
where we employ $\underline{(\cdot)}$ to indicate the lower bound and $\overline{(\cdot)}$ the upper bound unless otherwise specified. It should be noted that this model can generalize to 3-D motions by rotating the coordinate system, making it applicable for representing a variety of maneuvers. 

We can transform this model into a non-dimensional form where the gravity magnitude, $g$, as well as the vehicle mass, $m$, can disappear throughout the derivation. The non-dimensional variables of time and position required for this process are defined as follows:
\begin{align}
\hat{t}=\overline{\omega}t,\;\hat{x}=\overline{\omega}^{2}x/g,\text{ and }\hat{z}=\overline{\omega}^{2}z/g,
\end{align}
where a hat $\hat{(\cdot)}$ is used to differentiate the non-dimensional states from the original states. We use overhead dots to represent time derivatives, e.g., $\dot{\hat{x}}=d\hat{x}/d\hat{t}$. This leads to the state and control vectors given below:
\begin{align}\hat{\mathbf{x}} & =[\hat{x},\dot{\hat{x}},\hat{z},\dot{\hat{z}},\theta]^{T}\in\mathbb{R}^{5},\\
\mathbf{\hat{u}} & =[u_{R},u_{T}]^{T}=[\frac{\omega}{\overline{\omega}},\frac{F_{T}}{mg}]^{T}\in \boldsymbol{U},
\end{align}
where $\boldsymbol{U}$ is the set of all admissible controls satisfying:
\begin{equation}
|u_{R}|\leq1 \text{ and }\frac{\underline{F_{T}}}{mg}=\underline{u_{T}}\leq u_{T}\leq\overline{u_{T}}=\frac{\overline{F_{T}}}{mg}.
\end{equation}
Note that $\theta$ and $\hat{\theta}$ are identical as they represent the same value. Finally, we obtain the following equation of motion:
\begin{align}
\dot{\hat{\mathbf{x}}}\!=\!\mathbf{f}(\hat{\mathbf{x}},\hat{\mathbf{u}})\!=\![\dot{\hat{x}},u_{T}\sin\theta,\dot{\hat{z}},u_{T}\cos\theta-1,u_{R}]^{T}.\label{equ_nondim_model}
\end{align}

\subsubsection{The Maximum Principle}
We leverage the PMP to derive the necessary optimality conditions as well as basic optimal control properties. As per the maximum principle \cite{liberzon2011calculus}, the optimal trajectory in a fixed-endpoint TOC problem must adhere to the following criteria: (i) satisfaction of the adjoint equation, (ii) satisfaction of the Hamiltonian maximization condition, and (iii) having a zero Hamiltonian along the optimal trajectory, i.e., $H\equiv0$ where $H$ is the Hamiltonian. Now we can apply these conditions to our model. 

Let's denote the time-optimal state trajectory as $\hat{\mathbf{x}}^{*}:[0,\hat{T}^{*}]\!\rightarrow\!\mathbb{R}^{5}$, or equivalently as the corresponding control inputs $\hat{\mathbf{u}}^{*}:[0,\hat{T}^{*}]\!\rightarrow\!\boldsymbol{U}$, where $\hat{T}^{*}$ is the optimal trajectory duration scaled by the maximal rotational rate. The maximization condition implies that:
\begin{equation}
H^{*}(\hat{t})=0,\;\forall \hat{t} \in [0,\hat{T}^{*}],\label{equ_hamiltonian}
\end{equation}
where $H^{*}(\hat{t})\!:=\!H(\hat{\mathbf{x}}^{*}(\hat{t}),\hat{\mathbf{u}}^{*}(\hat{t}),\mathbf{p}^{*}(\hat{t}))\!=\!\underset{\hat{\mathbf{u}}\in\boldsymbol{U}}{\max}\;H(\hat{\mathbf{x}}^{*},\hat{\mathbf{u}},\mathbf{p}^{*})$. Since the running cost for TOC is one, we can arrive at the following Hamiltonian for the quadrotor system:
\begin{align}
\begin{split}
&H(\hat{\mathbf{x}},\hat{\mathbf{u}},\mathbf{p})=\mathbf{p}^{T}\mathbf{f}(\hat{\mathbf{x}},\hat{\mathbf{u}})+1  \\=p_{1}\dot{\hat{x}}\!+\!p_{2}u_{T}&\sin\theta\!+\!p_{3}\dot{\hat{z}}
\!+\! p_{4}(u_{T}\cos\theta-1)\!+\!p_{5}u_{R}\!+\!1,
\end{split}\label{equ_hamiltonian_details}
\end{align}
where $\mathbf{p}\!=\![p_{1},p_{2},p_{3},p_{4},p_{5}]^{T}$ is the adjoint state vector (or co-states) that satisfies the adjoint equation below:
\begin{equation}
\mathbf{\dot{\mathbf{p}}}=-\nabla_{\hat{\mathbf{x}}}H(\hat{\mathbf{x}}^{*},\hat{\mathbf{u}}^{*},\mathbf{p}).\label{equ_adjoint_equation}
\end{equation}
By taking the derivative of the Hamiltonian in Eq. (\ref{equ_adjoint_equation}) w.r.t. $\hat{\mathbf{x}}$, we get the exact expressions of the first four adjoint states:
\begin{subequations}\label{equ_costates}
\begin{align}
p_{1}=c_{1},\;p_{2}=c_{2}-c_{1}\hat{t},\\
p_{3}=c_{3},\;p_{4}=c_{4}-c_{3}\hat{t},
\end{align}
\end{subequations}
where $c_{1},c_{2},c_{3},c_{4}$ are their coefficients which are constant. We introduce the vector $\mathbf{c}$ as $\mathbf{c}\!=\![c_{1},c_{2},c_{3},c_{4}]^{T}$. Meanwhile, the differential equation of $p_{5}$ is offered:
\begin{equation}
\dot{p}_{5}=-p_{2}u_{T}^{*}\cos\theta^{*}+p_{4}u_{T}^{*}\sin\theta^{*}.\label{equ:p5_dot}
\end{equation}
Since $\mathbf{p}$ must be nontrivial in time-optimal trajectories \cite{sussmann1991shortest}, i.e., $\mathbf{p}(\hat{t})\neq\mathbf{0}$ for every $\hat{t}\in\mathbb{R}$, we know that $\mathbf{c}$ cannot be all zeros, i.e., $\mathbf{c}\! \neq \!\mathbf{0}$. In addition, we observe that:
\begin{lemma}\label{lem:p2p4}
$p_{2}$ and $p_{4}$ have at most one isolated zero.
\end{lemma}
\begin{proof}
This is obvious from the fact that $p_{2}$ and $p_{4}$ are both affine functions of $\hat{t}$ and they never vanish.
\end{proof}

Now, given that optimal controls must maximize the Hamiltonian, we can express the optimal control policy as follows:
\begin{align}
u_{T}^{*}&=\underset{u_{T}\in[\underline{u_{T}},\overline{u_{T}}]}{\arg\max}\Phi_{T}u_{T},\label{equ_utopt}\\
u_{R}^{*}&=\underset{u_{R}\in[-1,+1]}{\arg\max}\Phi_{R}u_{R},\label{equ_uropt}
\end{align}
where $\Phi_{T}$ and $\Phi_{R}$ are \textit{switching functions} for the collective thrust and rotational rate, respectively, which can be obtained by combining terms related to control inputs in Eq. (\ref{equ_hamiltonian_details}):
\begin{equation}
\Phi_{T}=p_{2}\sin\theta^{*}+p_{4}\cos\theta^{*} \text{ and } \Phi_{R}=p_{5}. \label{equ_phi_T}
\end{equation}
Note that $\dot{\Phi}_{R}=\dot{p_{5}}$ also holds. Finally, expanding the optimal control policy gives:
\begin{equation}
u_{T}^{*}\!=\!\begin{cases}
\begin{array}{c}
\overline{u_{T}}\ \text{if }\Phi_{T}\!>\!0\\
\underline{u_{T}}\ \text{if }\Phi_{T}\!<\!0
\end{array} & ,\end{cases}u_{R}^{*}\!=\!\begin{cases}
\begin{array}{c}
+1\ \text{if }\Phi_{R}\!>\!0\\
-1\ \text{if }\Phi_{R}\!<\!0
\end{array} & .\end{cases}
\end{equation}
It is obvious that $u_{T}$ and $u_{R}$ will switch their values when the corresponding switching functions change signs or vanish (i.e., become zero identically). 

We recap some time-optimal-control-related terms used in this paper. A control policy is called \textit{bang-bang} if it only takes on either minimum or maximum values throughout the mission, or equivalently, its switching function never vanishes. A trajectory segment is rendered a \textit{bang arc} for a certain control if the corresponding switching function contains no zeros over $(a,b)$ where $a$ and $b$ are switching times. Moreover, we use ``a control is B-B-2 over $[a,b]$'' to refer that this control is bang-bang with two switches inside. A control is \textit{singular} if the corresponding switching function vanishes within a compact interval $[c,d]$. Note that this segment is deemed a \textit{singular arc} if $c$ and $d$ are switching times. Furthermore, a control policy is called \textit{bang-singular} if it contains both bang and singular arcs, and we say ``a control is B-S-B over $[a,b]$'' if this control consists of three segments following an order of bang, singular, and bang within the interval. It is worth mentioning that in our definitions, a switch can refer to control input changes between two bang arcs, as well as between a bang arc and a singular arc.

The main conclusions from \cite{hehn2012performance} are summarized in the following lemma:
\begin{lemma}\label{lem:prior_works}
There exist time-optimal trajectories for quadrotors. Moreover, $u_{T}^{*}$ is bang-bang and $u_{R}^{*}$ is bang-singular. Additionally, the optimal rotational rate input during a singular arc can be expressed as:
\begin{equation}
u_{R}^{*}=\frac{c_{2}c_{3}-c_{1}c_{4}}{(c_{1}^{2}+c_{3}^{2})\hat{t}^{2}-2(c_{1}c_{2}+c_{3}c_{4})\hat{t}+c_{2}^{2}+c_{4}^{2}}.\label{equ_ur_sing}
\end{equation}
\end{lemma}
\begin{proof}
See \cite{hehn2012performance}.
\end{proof}

Lemma \ref{lem:prior_works} helps explain control behavior during optimal maneuvers, but it is insufficient to fully characterize the structure of optimal trajectories. Therefore, we introduce a concept of singular flows to enable a more in-depth analysis.

\subsection{Singular Flows and Their Properties}

We observe that unlike $\Phi_{T}$, the zeros of $\Phi_{R}$ are quite difficult to study as only the expression of its time derivative is available. Fortunately, $\dot{\Phi}_{R}\equiv0$ is a necessary condition for $\Phi_{R}\equiv0$. Therefore, by studying $\theta$ that can satisfy $\dot{\Phi}_{R}\equiv0$, it is still possible to gain insights into conditions a singular control must obey. This necessitates the introduction of singular flows:

\begin{definition}[Singular Flow]\label{def:vanishing_curve}
Given $(p_{2}, p_{4})\!\neq\!(0,0)$ for all times, a singular flow $\varTheta$ is an absolutely continuous function that satisfies $-p_{2}(\hat{t})\cos\varTheta(\hat{t})\!+\!p_{4}(\hat{t})\sin\varTheta(\hat{t})\!=\!0$ for all $\hat{t}\in\mathbb{R}$. Moreover, we assume that $|\dot{\varTheta}|\!<\!1$ holds almost everywhere.
\end{definition}

It is not difficult to find that the constraint related to singular flows is equivalent to condition $\dot{\Phi}_{R}=0$. The proof is straightforward: because $u_{T}\!>\!0$, we can divide $u_{T}^{*}$ at both sides of Eq. (\ref{equ:p5_dot}), which reveals that any $\theta^{*}$ that meets $\dot{\Phi}_{R}=0$ will also satisfy $-p_{2}\cos\theta^{*}\!+\!p_{4}\sin\theta^{*}\!=\!0$.

Singular flows are closely associated with the behavior of $u_{R}^{*}$ during a singular arc:

\begin{lemma}\label{lem:ur_singular}
$u_{R}^{*}$ is singular if and only if the corresponding $\theta^{*}$ lies on a singular flow $\varTheta$.
\end{lemma}
\begin{proof}
``$\Rightarrow$'': The singularity of of $u_{R}^{*}$ implies that $\Phi_R\!=\!\dot{\Phi}_R \!\equiv \!0$ holds. Knowing that $p_{2}$, $p_{4}$ have at most one isolated zero by Lemma \ref{lem:p2p4}, the only possibility to fulfill $\dot{\Phi}_R \equiv 0$ for a nontrivial interval is by having a smooth trajectory of $\theta^{*}$ satisfy $p_{4}\sin\theta^{*}\!-\!p_{2}\cos\theta^{*}\!=\!0$, which suggests that $\theta^{*}=\varTheta$ and therefore gives the statement.

``$\Leftarrow$'': $\theta^{*}$ lies on $\varTheta$ $\Rightarrow$ $\theta^{*}\equiv\varTheta$ within that interval $\Rightarrow$ $|u_{R}^{*}|\equiv|\dot{\varTheta}|<1$ $\Rightarrow$ $u_{R}^{*}$ is singular.
\end{proof}

\begin{remark}
Note that singular flows are not unique. This can be confirmed by checking its analytical expression, i.e.,  $\varTheta(\hat{t})\!=\!\arctan(p_{2}(\hat{t})/p_{4}(\hat{t}))\!\pm\! k\pi$, $k\in\mathbb{N}_{0}$, which discloses that different $k$ results in different singular flow. Moreover, two neighboring singular flows have a constant shift of $\pi$. In the remainder of this paper, we say ``two singular flows are \textit{$k$-adjacent}'' if their gap is $k\pi$. Fig. \ref{fig_vanishing_curve} gives an example of singular flows in a rest-to-rest purely horizontal translation, illustrating their functionality in analysis.
\end{remark}

\begin{figure}[!htbp]
\centering
\includegraphics[width=0.45\textwidth]{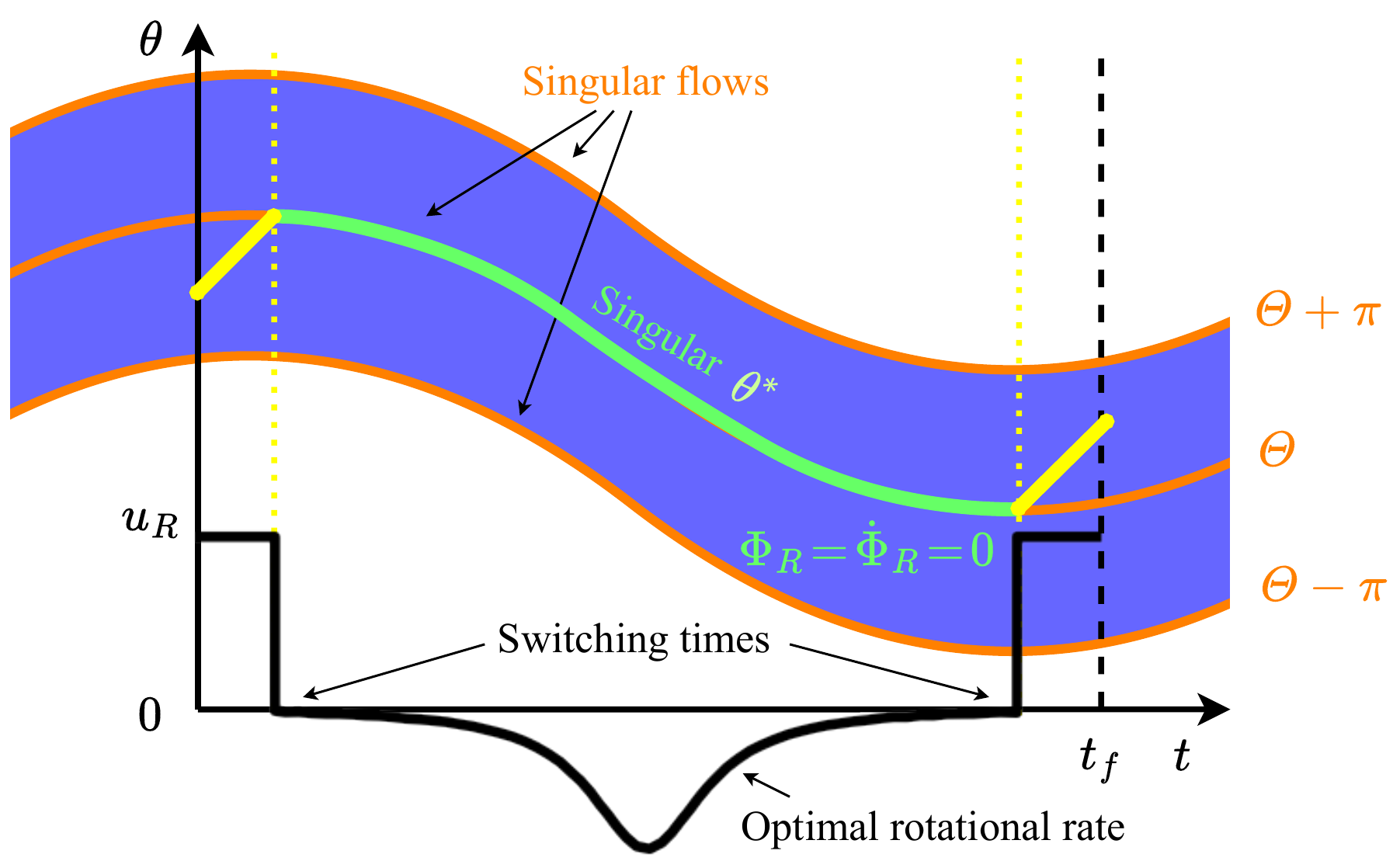}
\caption{Demonstration of singular flows (orange lines) in a time-optimal maneuver of a pure horizontal translation. Note that the optimal rotation trajectory $\theta^{*}$ consists of three phases, represented by two yellow lines and one green curve in the middle. The corresponding optimal rotational rate input $u_{R}^{*}$ is depicted at the bottom trajectory (black line). In the first phase, the vehicle pitches forward with a maximal value of $u_{R}^{*}\!=\!1$, which corresponds to a bang arc (yellow line). Once it intersects the singular flow, we have $\theta^{*}\!=\!\varTheta$ for a long time where $\Phi_{R}\!=\!\dot{\Phi}_{R}\!=\!0$ holds, implying that $u_{R}^{*}$ is singular. This phase entails a smooth transition from acceleration to deceleration of the vehicle. In the last phase, the vehicle pitches forward again with $u_{R}\!=\!1$ to return a hover state. We see that by drawing singular flows, we know when $u_{R}^{*}$ can be singular and what the structure of optimal maneuvers looks like.} \label{fig_vanishing_curve}
\end{figure}

\begin{definition}[Flat Singular Flow]\label{def:straight_vanishing_curve}
If $\varTheta$ is a constant, then all related singular flows are flat. 
\end{definition}

As shown in Fig. \ref{fig_flip_maneuver}, a flat singular flow implies that $u_{R}^{*}$ is zero when singular. Now we study some implicit conditions that can potentially make a singular flow flat.

\begin{lemma}\label{lem:vartheta}
A singular flow is flat if and only if $c_{2}c_{3}\!=\!c_{1}c_{4}$. Moreover, the corresponding $\varTheta$ has two solutions:
\begin{equation}
\varTheta=\begin{cases}
\begin{array}{c}
\arctan(c_{2}/c_{4})\pm k\pi,\\
\arctan(c_{1}/c_{3})\pm k\pi,
\end{array} & \begin{array}{c}
c_{1}=c_{3}=0,\\
\text{otherwise},
\end{array}\end{cases}\label{equ:straight_vanishing_curve}
\end{equation}
where $k\in\mathbb{N}_{0}$.
\end{lemma}
\begin{proof}
``$\Rightarrow$'': Writing down the expression of $\dot{\varTheta}$ by using the constraint given in Definition \ref{def:vanishing_curve}, we find that $\dot{\varTheta}$ shares the same expression as $u_{R}^{*}$ in Eq. (\ref{equ_ur_sing}). Since $\varTheta$ is a constant, $\dot{\varTheta}(\hat{t})$ must be zero for all $\hat{t}\in\mathbb{R}$. This can only be met when the numerator $c_{2}c_{3}\!-\!c_{1}c_{4}\!=\!0$, and hence confirm the statement.

``$\Leftarrow$'': Inserting $c_{2}c_{3}\!=\!c_{1}c_{4}$ into the expression of $\dot{\varTheta}$ yields that $\varTheta$ is a constant $\Rightarrow$ the singular flow is flat.

Now let's solve $\varTheta$. Assume that $c_{1}\!=\!c_{3}\!=\!0$. This implies that $(c_{2},c_{4})\!\neq\!(0,0)$ because otherwise $\mathbf{c}\neq\mathbf{0}$ cannot be met. Using the constraint, it is trivial to have $\varTheta=\arctan(c_{2}/c_{4})\pm k\pi$, which gives the first solution family. For the second solution family, we need to take the time derivative on both sides of the constraint, which leads to:
\begin{equation}
(c_{1}\cos\varTheta-c_{3}\sin\varTheta)+\Phi_{T}\dot{\varTheta}=0.
\end{equation}
Since $\dot{\varTheta}\!\equiv\!0$, we have
\begin{equation}
c_{1}\cos\varTheta-c_{3}\sin\varTheta=0,\label{equ:lemma2_proof}
\end{equation}
which eventually suggests that $\varTheta=\arctan(c_{1}/c_{3})\pm k\pi$.
\end{proof}

\begin{remark}
Lemma \ref{lem:vartheta} establishes the relation between the flatness of the singular flow and the configuration of the adjoint state coefficients $\mathbf{c}$. It is worth mentioning that since the denominator of $\dot{\varTheta}$ shown in Eq. (\ref{equ_ur_sing}) is always nonzero (this can be easily verified by deriving its roots with $\mathbf{c}\!\neq\!\mathbf{0}$), we obtain that $u_{R}^{*}$ must be nonzero provided $c_{2}c_{3}\!-\!c_{1}c_{4}\!\neq\!0$. Using Lemma \ref{lem:vartheta}, we further reach the conclusion that if $u_{R}^{*}$ is zero at any time, the corresponding singular flows must be flat, and hence $u_{R}^{*}$ must be zero identically when singular.
\end{remark}

\subsection{Elementary Properties of Optimal Trajectories}

This section investigates inherent properties of optimal trajectories that shed light on the bang-bang structure of $u_{T}^{*}$ and the bang-singular structure of $u_{R}^{*}$.

\begin{lemma}\label{lem:phi_p2p4}
The following conditions are equivalent: 
\begin{enumerate}[label=(\roman{*})]
\item there exist an unique $\hat{t}\in\mathbb{R}$ such that $\Phi_{T}(\hat{t})\!=\!\dot{\Phi}_{R}(\hat{t})\!=\!0$;
\item $c_{2}c_{3}\!=\!c_{1}c_{4}$ and $(c_{1},c_{3})\neq(0,0)$;
\item $p_{2}(\hat{t})\!=\!p_{4}(\hat{t})\!=\!0$.
\end{enumerate}
\end{lemma}
\begin{proof}
(i)$\Rightarrow$(ii): Given $\Phi_{T}(\hat{t})\!=\!\dot{\Phi}_{R}(\hat{t})\!=\!0$, solving $\theta^{*}$ meeting both equations yields:
\begin{equation}
\theta^{*}=\arctan\left(-\frac{p_{4}}{p_{2}}\right)=\arctan\left(\frac{p_{2}}{p_{4}}\right).
\end{equation}
Taking the tangent of both sides and multiplying the constraint out results in $p_{4}^{2}+p_{2}^{2}=0$, expending which gives:
\begin{equation}
(c_{1}^{2}+c_{3}^{2})\hat{t}^{2}-2(c_{1}c_{2}+c_{3}c_{4})\hat{t}+c_{2}^{2}+c_{4}^{2}=0.\label{equ_p2_01}
\end{equation}
It should be noticed that the left-hand side is a convex function in $\hat{t}$, and since $c_{1}^{2}+c_{3}^{2}>0$, it has a global minimum. To justify Eq. (\ref{equ_p2_01}), the global minimum must be non-positive. We compute the global minimizer $\hat{t}^{*}$ through the first-order optimality condition:
\begin{equation}
\hat{t}^{*}=(c_{1}c_{2}+c_{3}c_{4})/(c_{1}^{2}+c_{3}^{2}).\label{equ_tmin}
\end{equation}
Substituting $\hat{t}^{*}$ into Eq. (\ref{equ_p2_01}) yields the global minimum:
\begin{align}
\begin{split}
&-(c_{1}c_{2}+c_{3}c_{4})^{2}/(c_{1}^{2}+c_{3}^{2})+c_{2}^{2}+c_{4}^{2}\\
=&(c_{2}c_{3}-c_{1}c_{4})^{2}/(c_{1}^{2}+c_{3}^{2})\leq0.
\end{split}
\end{align}
It is clear that $c_{2}c_{3}-c_{1}c_{4}$ must be zero to have the above inequality hold. To complete the proof, we still need to show that $(c_{1},c_{3})\neq(0,0)$. This can be easily validated by assuming $c_{1}\!=\!c_{3}\!=\!0$ and inserting it into Eq. (\ref{equ_p2_01}), which yields $c_{2}^{2}\!+\!c_{4}^{2}\!=\!0$ $\Rightarrow$ $c_{2}\!=\!c_{4}\!=\!0$. However, this is impossible because $\mathbf{c}\neq\mathbf{0}$.

(ii)$\Rightarrow$(iii): Given $(c_{1},c_{3})\neq(0,0)$, we have $c_{2}c_{3}\!=\!c_{1}c_{4}$ $\Rightarrow$ $c_{2}/c_{1}\!=\!c_{4}/c_{3}\!=\!a$ for some $a\in\mathbb{R}$ $\Rightarrow$ $c_{2}\!-\!c_{1}a\!=\!c_{4}\!-\!c_{3}a\!=\!0$ $\Rightarrow$ $p_{2}(a)\!=\!p_{4}(a)\!=\!0$. According to Lemma \ref{lem:p2p4}, $a$ is unique and thus $a=\hat{t}$.

(iii)$\Rightarrow$(i): The proof is trivial by inserting $p_{2}\!=\!p_{4}\!=\!0$ into $\Phi_{T}$ and $\dot{\Phi}_{R}$, respectively. The uniqueness of $\hat{t}$ follows from Lemma \ref{lem:p2p4}.
\end{proof}

\begin{remark}
Lemma \ref{lem:phi_p2p4} indicates that if $p_{2}(\hat{t})\!=\!p_{4}(\hat{t})\!=\!0$ holds at any time, the corresponding singular flows must be flat.
\end{remark}

\begin{corollary}\label{cor:phir_0}
Consider a time $\hat{t}\in\mathbb{R}$. We have $\dot{\Phi}_R(\hat{t})=0$ if and only if either $p_{2}(\hat{t})\!=\!p_{4}(\hat{t})\!=\!0$ or $\theta^{*}(\hat{t})=\varTheta(\hat{t})$ hold.
\end{corollary}
\begin{proof}
Applying Lemma \ref{lem:phi_p2p4} and the definition of singular flows to Eq. (\ref{equ:p5_dot}) leads to the statement.
\end{proof}

\begin{lemma}\label{lem:c1c30}
If $u_{R}^{*}$ contains multiple singular arcs, then $c_{1}$ and $c_{3}$ cannot be both zero.
\end{lemma}
\begin{proof}
Assume that $u_{R}^{*}$ contains two singular arcs and $c_{1}\!=\!c_{3}\!=\!0$ holds (and therefore $c_{2}c_{3}\!=\!c_{1}c_{4}$). It follows that $u_{R}^{*}\!=\!0$ during both singular arcs since the corresponding singular flow is flat by Lemma \ref{lem:vartheta}. Let's focus on the first singular arc. We know that $\Phi_{R}\!=\!\dot{\Phi}_{R}=0$ holds when $u_{R}^{*}$ is singular. Inserting $c_{1}=c_{3}=0$ into the expression of $\dot{\Phi}_{R}$ and dividing out $u_{T}^{*}\!>\!0$, we get: 	
\begin{equation}
\dot{\Phi}_{R}=-c_{2}\cos\theta^{*}+c_{4}\sin\theta^{*}=0.
\end{equation}
It is obvious that $\theta^{*}$ is the only variable that is manipulated by $u_{R}^{*}$. It suggests that this singular arc will last until the end time as the corresponding $\Phi_{R}$ will always be zero. Therefore, it is impossible to have a second singular arc. 
\end{proof}

\begin{lemma}\label{lem:phit_0}
Consider a singular flow $\varTheta$ and a time $\hat{t}\in\mathbb{R}$. If $\theta^{*}(\hat{t})=\varTheta(\hat{t})\pm \frac{\pi}{2}$ holds, then $\Phi_{T}(\hat{t})\!=\!0$. Moreover, the converse holds if $(p_{2}(\hat{t}),p_{4}(\hat{t}))\neq(0,0)$. 
\end{lemma}
\begin{proof}
``$\Rightarrow$'': Since $\varTheta$ is arbitrary and any two neighboring singular flows have a gap of $\pi$, we only need to prove the case of $\theta^{*}(\hat{t})=\varTheta(\hat{t})-\frac{\pi}{2}$ and set $k$ as zero, which leads to a solution of $\varTheta(\hat{t})\!=\!\arctan(p_{2}(\hat{t})/p_{4}(\hat{t}))$. We omit the situation where $p_{2}(\hat{t})\!=\!p_{4}(\hat{t})\!=\!0$  as it trivially results in $\Phi_{T}(\hat{t})=0$. When $\theta^{*}$ and $\varTheta-\frac{\pi}{2}$ intersect, we have the following equation:
\begin{equation}
\theta^{*}(\hat{t})=\arctan(\frac{p_{2}(\hat{t})}{p_{4}(\hat{t})})-\frac{\pi}{2}=\arctan(-\frac{p_{4}(\hat{t})}{p_{2}(\hat{t})}).\label{equ_lem5_theta}
\end{equation}
Herein, the identity $\arctan(a)+\arctan(-\!1/a)\!=\!\pi/2$ is used where $a\in\mathbb{R}$. By taking tangents on both sides and rearranging the result, we arrive at:
\begin{equation}
p_{2}(\hat{t})\sin\theta^{*}(\hat{t})+p_{4}(\hat{t})\cos\theta^{*}(\hat{t})=0 \Rightarrow \Phi_{T}(\hat{t})=0,
\end{equation}
which verifies the claim.

``$\Leftarrow$'': The proof is skipped as it is trivial.
\end{proof}

\begin{remark}
Corollary \ref{cor:phir_0} and Lemma \ref{lem:phit_0} enable us to derive the zeros of $\dot{\Phi}_{R}$ and $\Phi_{T}$ by drawing the corresponding singular flows, as illustrated in Fig. \ref{fig_flip_maneuver}. More specifically, we can obtain the maximum number of their isolated zeros by simply counting the intersections between $\theta^{*}$ and singular flows.
\end{remark}

\begin{corollary}\label{cor:ut_switch}
Consider a flat singular flow $\varTheta$ and a time $\hat{\tau}\in(0,\hat{T}^{*})$. $\hat{\tau}$ is a switching time for $u_{T}^{*}$ if and only if either of the following conditions holds:
\begin{enumerate}[label=(\roman{*})]
\item $p_{2}(\hat{\tau})=p_{4}(\hat{\tau})=0$ and $\theta^{*}(\hat{\tau})\neq\varTheta(\hat{\tau})\pm\frac{\pi}{2}$;
\item $\theta^{*}(\hat{\tau})=\varTheta(\hat{\tau})\pm\frac{\pi}{2}$ and $(p_{2}(\hat{\tau}),p_{4}(\hat{\tau}))\neq(0,0)$.
\end{enumerate}
\end{corollary}
\begin{proof}
See Appendix A.
\end{proof}

\begin{corollary}\label{cor:ut_switch_hori}
Consider $c_{2}c_{3}\!\neq\!c_{1}c_{4}$. If a time $\hat{\tau}\in[0,\hat{T}^{*}]$ is a switching time, then $\theta^{*}(\hat{\tau})=\varTheta(\hat{\tau})\pm\frac{\pi}{2}$.
\end{corollary}
\begin{proof}
Firstly, we obtain $(p_{2}(\hat{\tau}),p_{4}(\hat{\tau}))\neq(0,0)$ from $c_{2}c_{3}\!\neq\!c_{1}c_{4}$ using Lemma \ref{lem:phi_p2p4}. To make $\hat{\tau}$ a switching time, we must have $\Phi_{T}=0$. As per Lemma \ref{lem:phit_0}, this is possible only if $\theta^{*}(\hat{\tau})$ equals $\varTheta(\hat{\tau})\pm\frac{\pi}{2}$, which completes the proof.
\end{proof}

\begin{lemma}\label{lem:ur_notbb}
If $(c_{1},c_{3})\neq(0,0)$, then $u_{R}^{*}$ cannot be bang-bang on any interval.
\end{lemma}
\begin{proof}
See Appendix B.
\end{proof}

\begin{remark}
Lemma \ref{lem:ur_notbb} indicates that in a time-optimal maneuver, the quadrotor will not oscillate back and forth with an extremal rotational rate. This property rules out a large set of non-optimal trajectories. For instance, aided with Lemma \ref{lem:c1c30}, we can deduce that if $u_{R}^{*}$ contains two singular arcs, then its structure must be S-B-S rather than S-B-B-S.
\end{remark}

\begin{lemma}\label{lem:ur_sbs}
Assume that $(c_{1},c_{3})\neq(0,0)$. If $u_{R}^{*}$ is S-B-S over $[a,b]$ and the corresponding bang arc occurs on $(c,d)$ where $a<c<d<b$, then the following statements are true:
\begin{enumerate}[label=(\roman{*})]
\item these two singular arcs are not on the same singular flow;
\item if these two singular arcs are adjacent, then there must exist a $\hat{\tau}\in(c,d)$ such that $p_{2}(\hat{\tau})=p_{4}(\hat{\tau})=0$;
\item if $c_{2}c_{3}=c_{1}c_{4}$ holds and these two singular arcs are 2-adjacent, then there must exist a time $\hat{\tau}\in(c,d)$ such that $p_{2}(\hat{\tau})=p_{4}(\hat{\tau})=0$.
\end{enumerate}
\end{lemma}
\begin{proof}
See Appendix C.
\end{proof}

\begin{remark}
Lemma \ref{lem:ur_sbs} indicates the necessary conditions for the policy of $u_{R}^{*}$ to be S-B-S. It should be noted that among these conditions, the most critical one is the existence of a $\hat{\tau}$ such that $p_{2}(\hat{\tau})\!=\!p_{4}(\hat{\tau})\!=\!0$. If we manage to deny the existence of such a $\hat{\tau}$, then a large amount of candidate trajectories can be deemed non-optimal and removed from the solution set. 
\end{remark}

\subsection{Properties of Time-Optimal Quadrotor Maneuvers}

Using the derived properties, we can now investigate the characteristics of the time-optimal quadrotor maneuver. 

It is well-known that the necessary condition of the PMP is often too weak and is satisfied by too many trajectories other than the optimal ones. Generally speaking, only solutions with regularity properties\footnote{regularity means that the control has finitely many switches} are useful, but proving it for all the solutions is nearly impossible \cite{sussmann1991shortest}. To make the derivation tractable, we narrow our attention to a family of tasks whose solution sets meet specific conditions, followed by showing that the chosen family is large enough to encompass all flights of interest. We name the selected task family after nontrivial time-optimal maneuvers, which are defined as follows:

\begin{definition}[Nontrivial Time-Optimal Maneuver]\label{def:nontrivial_flight}
A time-optimal maneuver is considered nontrivial if its time-optimal trajectory satisfies the following conditions:
\begin{enumerate}[label=(\roman*)]
\item there exists a central singular flow $\varTheta_{c}$ such that $\theta^{*}$ satisfies $\varTheta_{c}\!-\!\pi\leq\theta^{*}\leq\varTheta_{c}\!+\!\pi$ within $[0,\hat{T}^{*}]$;
\item $(c_{1},c_{3})\neq (0,0)$;
\item if $u_{R}^{*}$ is bang-singular, it has a singular arc on $\theta_{c}$.
\end{enumerate}
\end{definition}

The introduction of nontrivial time-optimal maneuvers allows us to concentrate on tasks whose $\theta^{*}$ are within two 2-adjacent singular flows (approx. a range of $2\pi$), as exemplified in Fig. \ref{fig_vanishing_curve} and Fig. \ref{fig_flip_maneuver}. This condition is quite broad, and it even includes maneuvers with a 360$^{\circ}$ flip. Note that the second and third conditions serve to exclude cases where $u_{R}^{*}$ has infinitely many switches (the proof is skipped)---cases that are not of interest to us.

To facilitate the discussion, we consider two adjoint state configurations separately: $c_{2}c_{3}\!=\!c_{1}c_{4}$ and $c_{2}c_{3}\!\neq\! c_{1}c_{4}$.

\begin{figure}[t]
\centering
\includegraphics[width=0.45\textwidth]{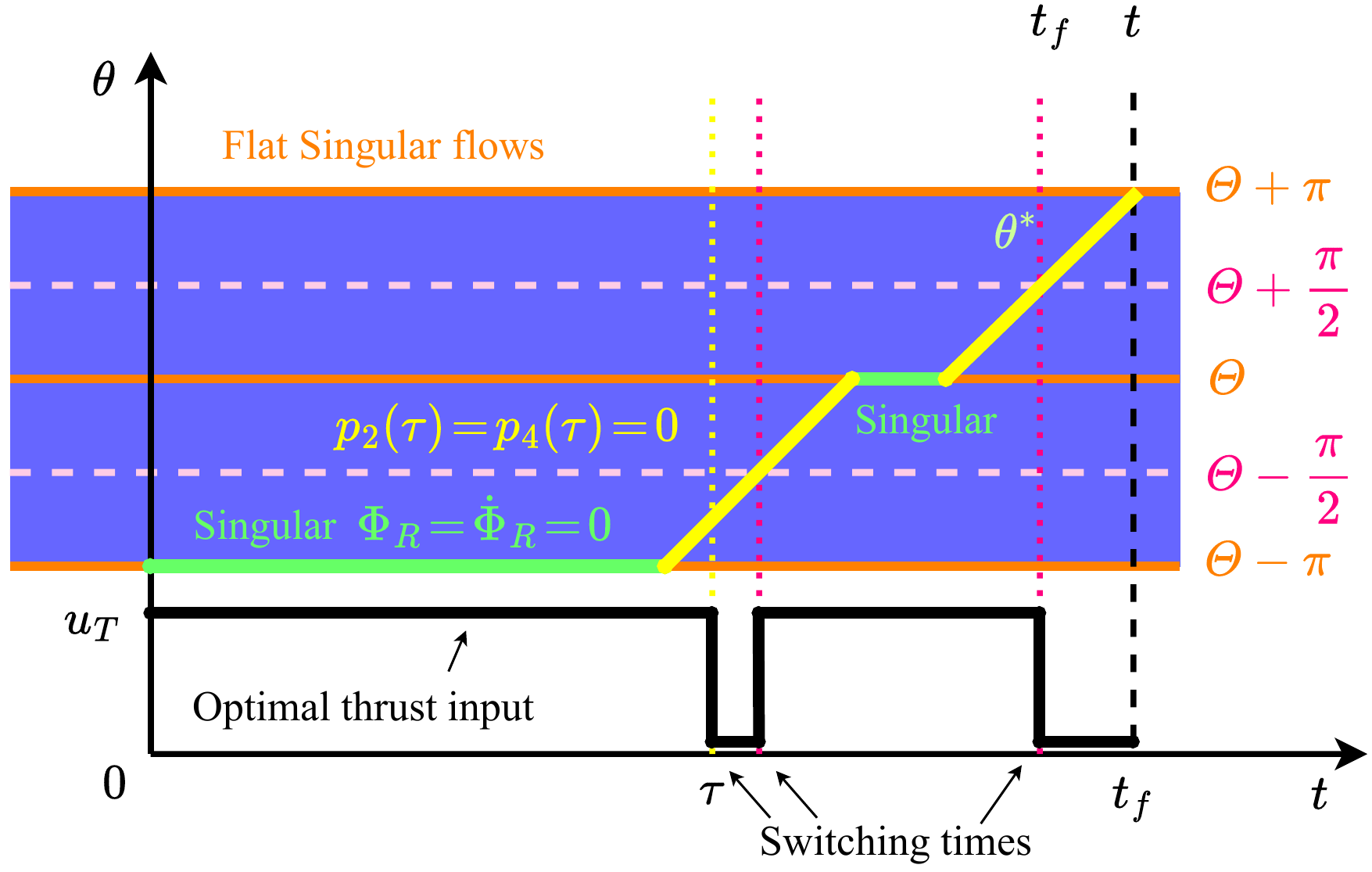}
\caption{Demonstration of how to use singular flows (orange lines) to determine the optimal thrust input, $u_{T}^{*}$ (black line),  for a purely vertical translation. It is known that the time-optimal maneuver entails a flip. By checking the intersections between the optimal rotation trajectory (green and yellow lines) and the singular flows shifted by $\pi/2$ as per Corollary \ref{cor:ut_switch}, we get the exact switching times of $u_{T}^{*}$. After taking into account one additional switch caused by $p_{2}(\hat{\tau})=p_{4}(\hat{\tau})=0$, we obtain the full structure of $u_{T}^{*}$ as shown in the bottom trajectory (black).} \label{fig_flip_maneuver}
\end{figure}

\begin{lemma}[]\label{lem:nonhorizontal}
Consider a nontrivial time-optimal maneuver with $c_{2}c_{3}\!=\!c_{1}c_{4}$. The following statements hold:
\begin{enumerate}[label=(\roman*)]
\item $u_{R}^{*}=0$ when singular;
\item $u_{T}^{*}$ is bang-bang with at most 5 switches, and $u_{R}^{*}$ is bang-singular with at most 2 singular arcs and 3 isolated bang arcs, following a structure of B-S-B-S-B.
\end{enumerate}
\end{lemma}
\begin{proof}
See Appendix D.
\end{proof}

\begin{lemma}[]\label{lem:horizontal}
Consider a nontrivial time-optimal maneuver with $c_{2}c_{3}\neq c_{1}c_{4}$. The following statements hold:
\begin{enumerate}[label=(\roman*)]
\item $(p_{2}(\hat{t}),p_{4}(\hat{t}))\neq(0,0)$ for all $\hat{t}\in\mathbb{R}$;
\item when $u_{R}^{*}$ is singular, it is nonzero and follows Eq. (\ref{equ_ur_sing}).
\item $u_{T}^{*}$ is bang-bang with at most 2 switches, and $u_{R}^{*}$ is bang-singular with at most 1 singular arc and 2 isolated bang arcs, following a structure of B-S-B.
\end{enumerate}
\end{lemma}

\begin{proof}
See Appendix E.
\end{proof}

\begin{remark}
Lemma \ref{lem:nonhorizontal} and Lemma \ref{lem:horizontal} 
are powerful verification tools for analytical and numerical results. They depict the most complicated control policy one may encounter in nontrivial time-optimal maneuvers. In most cases, the optimal trajectory can be characterized by a small portion of this policy. Take a purely vertical rest-to-rest flight as an example, whose singular flows and optimal rotation trajectory $\theta^{*}$ are shown in Fig. \ref{fig_flip_maneuver}. The first information we get is that the singular flows are flat, implying that $c_{2}c_{3}\!=\!c_{1}c_{4}$ and therefore Lemma \ref{lem:nonhorizontal} is applicable. We can easily check that the structure of $u_{R}^{*}$ is S-B-S-B, which is a portion of the extreme policy. By using Corollary \ref{cor:ut_switch}, we further know the exact switching times of $u_{T}^{*}$, which is key to obtain its structure of B-B-3. 
\end{remark}

Finally, we arrive at the following conclusion:
\begin{theorem}[]\label{the:nontrivial}
For nontrivial time-optimal maneuvers, the optimal collective thrust is bang-bang with at most 5 switches and the rotational rate is bang-singular with at most 4 switches.
\end{theorem}
\begin{proof}
Combine Lemma \ref{lem:nonhorizontal} and Lemma \ref{lem:horizontal}.
\end{proof}

\begin{remark}
Theorem \ref{the:nontrivial} specifies the maximum switch numbers for $u_{T}^{*}$ and $u_{R}^{*}$, which is essential for the proposed AOS to ensure time optimality under arbitrary start and end conditions. Its application will become clear in the next section.
\end{remark}

\section{Automatic Optimal Synthesis}\label{sec:aos}

This section elaborates on the process of AOS, including the construction of the optimization problem and the selection of polynomial families for state \& control parameterization.

\subsection{Preliminary}\label{subsec:preliminary}
\subsubsection{System Dynamics and Constraints}
Let $\mathcal{F}^{w}$ and $\mathcal{F}^{b}$ denote the world frame and the body frame, respectively. The state of the quadrotor is defined as 
\begin{equation}
\mathbf{x}=[\mathbf{p}^{w^{T}},\mathbf{q}_{wb}^{T},\mathbf{v}^{w^{T}},\boldsymbol{\omega}^{b^{T}}]^{T}\in\mathbb{R}^{n},
\end{equation}
where $\mathbf{p}^{w}\in\mathbb{R}^3$ is the position expressed in $\mathcal{F}^{w}$, $\mathbf{q}_{wb}\in SO(3)$ the unit quaternion rotation from $\mathcal{F}^{b}$ to $\mathcal{F}^{w}$, $\mathbf{v}^w\in\mathbb{R}^3$ the velocity w.r.t. $\mathcal{F}^{w}$, and $\boldsymbol{\omega}^{b}=[\boldsymbol{\omega}_{x}^{b},\boldsymbol{\omega}_{y}^{b},\boldsymbol{\omega}_{z}^{b}]^{T}\in \mathbb{R}^3$ the body-rates of the vehicle. The control input is the thrust command for each rotor:
\begin{equation}
\mathbf{u}=[f_{1},f_{2},f_{3},f_{4}]^{T}\in\mathbb{R}^{4},
\end{equation}
We refer readers to \cite{foehn2021time} for the detailed dynamics equation, $\dot{\mathbf{x}}=\mathbf{f}(\mathbf{x},\mathbf{u})$. Note that we will use $\underline{f}$ and $\overline{f}$ to denote the minimal and maximal thrusts for each rotor, respectively. In addition, the body-rate limit $\overline{\boldsymbol{\omega}}\!=\![\overline{\omega}_{s},\overline{\omega}_{s},\overline{\omega}_{\psi}]^{T}$ is enforced to align with the maximum sensor data rate of the gyroscope, $\overline{\omega}_{s}$, as well as the maximum yaw rate, $\overline{\omega}_{\psi}$. Eventually, we obtain the following constraints for the vehicle:
\begin{align}
&|\boldsymbol{\omega}_{x}^{b}|<\overline{\omega}_{s},|\boldsymbol{\omega}_{y}^{b}|<\overline{\omega}_{s},|\boldsymbol{\omega}_{z}^{b}|<\overline{\omega}_{\psi},\\
&\underline{f}\leq f_{i}\leq \overline{f},\;i=\{1,2,3,4\}.\label{equ_system_const}
\end{align}
For convenience, these constraints will be combined into a single inequality constraint denoted $\mathbf{h}(\mathbf{x},\mathbf{u})\leq\mathbf{0}$.

\subsubsection{Differential Flatness and Polynomial Representation}\label{subsec:df}
A system is \textit{differentially flat} \cite{levine2009analysis} if it is controllable and its states and controls can be parameterized by the so-called flat outputs and their finite orders of derivatives. The quadrotor's flat output can be formulated as position $\mathbf{p}^{w}$ and yaw angle $\psi$. In practice, we use polynomial representation to construct the flat output:
\begin{equation}
\mathbf{y}(t)=[\mathbf{p}^{w}(t),\psi(t)]^{T}\in\mathbb{R}^{4},
\end{equation}
Prior literature \cite{faessler2017differential, mellinger2011minimum, wang2022geometrically} has demonstrated that the states and control inputs can be retrieved from the flat output through the flatness maps, $\boldsymbol{\Psi}_{\mathbf{x}}$ and $\boldsymbol{\Psi}_{\mathbf{u}}$:
\begin{align}\mathbf{x} & =\boldsymbol{\Psi}_{\mathbf{x}}(\mathbf{y},...,\mathbf{y}^{(s-1)}),\\
\mathbf{u} & =\boldsymbol{\Psi}_{\mathbf{u}}(\mathbf{y},...,\mathbf{y}^{(s)}),
\end{align}
where $s$ is the derivative order of the polynomial. In the supplementary material\footnote{\url{https://drive.google.com/file/d/1fgbam0UkVmK9A-OsYeAzvCpCRnNDwqI_/view?usp=drive_link}}, we detail the derivation of the mapping from flat outputs to single-rotor thrust commands, along with their Jacobians. The treatments of the singularity in $\boldsymbol{\Psi}_{\mathbf{x}}$ and $\boldsymbol{\Psi}_{\mathbf{u}}$ are also discussed. To express equivalent state and input constraints with flat outputs, we introduce the vector $\mathbf{y}^{[s]}\!=\![\mathbf{y}^{T},\mathbf{\dot{y}}^{T},...,\mathbf{y}^{(s)^{T}}]^{T}$ which enables the following expression for constraints:
\begin{equation}
\mathbf{h}_{\boldsymbol{\Psi}}(\mathbf{y}):=\mathbf{h}(\boldsymbol{\Psi}_{\mathbf{x}}(\mathbf{y}^{[s-1]}),\boldsymbol{\Psi}_{\mathbf{u}}(\mathbf{y}^{[s]}))\leq\mathbf{0}.\label{equ_poly_const}
\end{equation}

Owing to the differential-flat property,  we are able to represent substantially long state and control trajectories without concerns about dynamical feasibility, since optimization over the flat-output space implicitly ensures the satisfaction of the dynamics equations of the original system.

\subsection{Problem Formulation}\label{subsec:methodology}
Consider a continuous-time TOC problem with fixed endpoints (or end-states). Let $\bar{\mathbf{x}}_{0}$ and $\bar{\mathbf{x}}_{f}$ denote the specified start and end states of the quadrotor, respectively. The TOC problem can be written as:
\begin{subequations}\label{equ_top}
\begin{align}\min_{\mathbf{x}(\cdot),\mathbf{u}(\cdot),t_{f}}\quad & \int_{t_{0}}^{t_{f}}1dt\\
\textrm{s.t.}\quad & \mathbf{x}(t_{0})=\bar{\mathbf{x}}_{0},\;\mathbf{x}(t_{f})=\bar{\mathbf{x}}_{f},\\
& \dot{\mathbf{x}}=\mathbf{f}(\mathbf{x},\mathbf{u}),\;\mathbf{h}(\mathbf{x},\mathbf{u})\leq\mathbf{0},
\end{align}
\end{subequations}
where $t_{0}$ and $t_{f}$ are the start and end times, respectively. By dividing the trajectory into $N$ pieces, we can transform Problem (\ref{equ_top}) into an equivalent $N$-stage problem:
\begin{subequations}\label{equ_top_multistage}
\begin{align}\min_{\mathbf{x}(\cdot),\mathbf{u}(\cdot),t_{1,...,N}}\quad & \sum_{k=1}^{N}\int_{t_{k-1}}^{t_{k}}1dt\\
\text{s.t.}\quad & \mathbf{x}(t_{0})=\bar{\mathbf{x}}_{0},\;\mathbf{x}(t_{N})=\bar{\mathbf{x}}_{f}\\
& \mathbf{x}_{k}(t_{k})=\mathbf{x}_{k+1}(t_{k}),\\
& \dot{\mathbf{x}}_{k}(t)\!=\!\mathbf{f}(\mathbf{x}_{k}(t),\mathbf{u}_{k}(t)),t\!\in\![t_{k-1},t_{k}],\\
& \mathbf{h}_{k}(\mathbf{x}_{k}(t),\mathbf{u}_{k}(t))\leq\mathbf{0},\;t\in[t_{k-1},t_{k}],\\
& t_{k-1}-t_{k}<0,
\end{align}
\end{subequations}
where $t_{k}$ is the timestamp for the $k$-th intermediate point and $t_{N}\!=\!t_{f}$. The equivalence of Problem \ref{equ_top_multistage} and Problem \ref{equ_top} is obvious from the principle of optimality \cite{diehl2011numerical}, i.e., each sub-trajectory of an optimal trajectory is also an optimal trajectory. 

Next, the flat output corresponding to each sub-trajectory is represented by a polynomial $\mathbf{y}_{k}:[0,T_{k}]\rightarrow \mathbb{R}^{4}$ and we denote $T_{k}$ as the trajectory duration for the $k$-th piece which is non-negative. To ensure state continuity, these polynomial pieces should be at least ($s\!-\!1$)-th order continuous at the junctions. This results in a piecewise polynomial with $N$ piece as the optimization variable. We then introduce $\bar{\mathbf{y}}_{0}^{[s-1]}$ and $\bar{\mathbf{y}}_{f}^{[s-1]}$ as the flat output at the start and end times, respectively, which are linked to $\bar{\mathbf{x}}_{0}$ and $\bar{\mathbf{x}}_{f}$ via the following formula:
\begin{equation}
\bar{\mathbf{x}}_{0}\!=\!\boldsymbol{\Psi}_{\mathbf{x}}(\bar{\mathbf{y}}_{0}^{[s-1]})\text{ and }\bar{\mathbf{x}}_{f}\!=\!\boldsymbol{\Psi}_{\mathbf{x}}(\bar{\mathbf{y}}_{f}^{[s-1]}).
\end{equation}
Now, the TOC problem parameterized by a piecewise polynomial can be constructed as:
\begin{subequations}\label{equ_top_poly}
\begin{align}\min_{\mathbf{y}_{1\!,...,N\!}}\quad & \sum_{k=1}^{N}T_{k}\\
\text{s.t.}\quad & \mathbf{y}_{1}^{[s-1]}(0)=\bar{\mathbf{y}}_{0}^{[s-1]},\;\mathbf{y}_{N}^{[s-1]}(T_{N})=\bar{\mathbf{y}}_{f}^{[s-1]}\\
& \mathbf{y}_{k}^{[s\!-\!1]}(0)=\mathbf{y}_{k-1}^{[s\!-\!1]}(T_{k}),\;k\!=\!2,...,N,\\
& \mathbf{h}_{\boldsymbol{\Psi}}(\mathbf{y}_{k})\leq\mathbf{0},\;-T_{k}<0.
\end{align}
\end{subequations}
It is important to note that Problem (\ref{equ_top_poly}) is no longer equivalent to (\ref{equ_top}), as we have restricted the solution space to certain polynomial families which does not necessarily include the real solutions. However, leveraging the fact that the time-optimal trajectory can be represented as a concatenation of multiple pieces, it is feasible to use a sufficiently large $N$ to maintain a low optimality gap. Meanwhile, we want to avoid making $N$ too large, as this would induce unnecessary computational burden. We will leave the discussion on selecting an appropriate $N$ for the next section.

There exists a highly efficient framework \cite{wang2022geometrically} for solving optimization problems that are structured in the form of Problem (\ref{equ_top_poly}). The requirements are that the order of each polynomial should be set to $(2s-1)$ and $\mathbf{h}_{\boldsymbol{\Psi}}(\mathbf{y})$ is differentiable, both of which are satisfied in our situation. After optimization, we convert the resulting flat output trajectory into a time-discretized trajectory $\mathbf{x}$ and $\mathbf{u}$ with the desired sampling time. Note that since $\psi$ is an independent variable, users can command the desired heading at each waypoint when generating the trajectory.

\subsection{Hypothesis about the Piece Number $N$ and Optimality}

\begin{figure}[!htbp]
\centering
\includegraphics[width=0.48\textwidth]{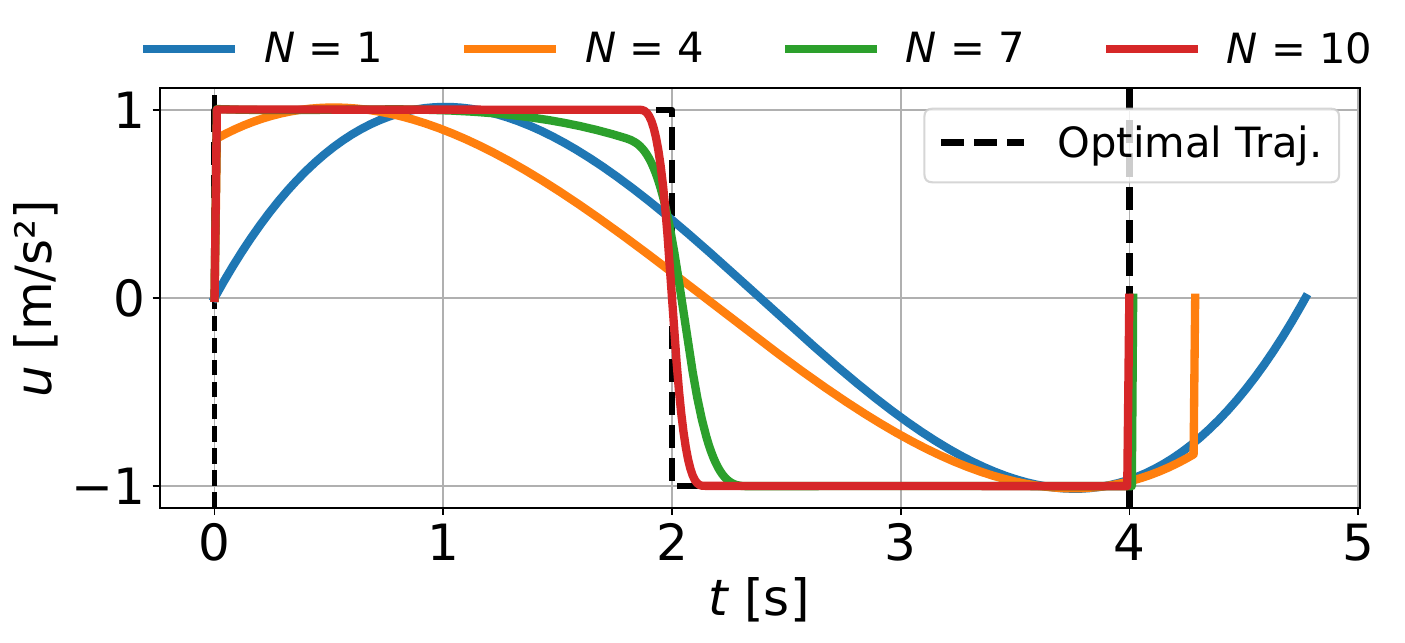}
\caption{Comparison of the resulting trajectories obtained from the proposed algorithm when applied to a double integrator ($\ddot{x}=u$ with $|u|\leq1$) with different polynomial numbers $N$. We set the initial state to $(x_{0},\dot{x}_{0})\!=\!(-2,0)$ and the terminal state to $(x_{f},\dot{x}_{f})\!=\!(0,0)$. The optimal trajectory (dashed line) can be obtained by analytical methods \cite{tedrake2009underactuated}, resulting in a duration of 4 s. To properly represent the flat output, the order of each polynomial is set to $2s\!-\!1$ where $s=3$. We notice that when $N\!=\!1$ (blue), AOS yields a timing of 5.96 s, which is 49\% longer than 4 s, and the corresponding $u$ can only touch the input limits twice. By increasing $N$ to 4 (orange) and 7 (green), the optimality gaps drop to 7\% and 1\%, respectively, and the corresponding $u$ are able to maintain extremal inputs for a substantially longer duration. When $N$ hits 10, the obtained trajectory exhibits a near bang-bang policy (red).}\label{fig_di}
\end{figure}

It is true that a single polynomial finds it challenging to fit a trajectory featuring multiple jumps in states or controls. However, as more pieces are introduced, the increased variable freedom enables better fitting quality, and a near-perfect performance can possibly be achieved as $N$ approaches infinity. We use an example to illustrate this idea. 
Consider a double integrator. We apply our AOS to generate its time-optimal trajectories with varying $N$ as shown in Fig. \ref{fig_di}. It shows that as $N$ increases, the piecewise polynomial quickly converges to the optimal trajectory, yielding an optimality gap of merely 1\% when $N\!=\!7$. This result showcases the possibility for polynomial representation to approach true time optimality. In the meantime, we notice that when $N$ is increased to certain points, such as 10 (approx. 1\% optimality gap), the improvement becomes quite marginal, yet it induces a larger computation time. To balance time optimality and computation cost, it is desirable to have a minimal $N$ that consistently yields an optimality gap below 10\%. Therefore, we raise the following hypothesis:
\begin{hyp}\label{hyp:optimal_n}
\textit{The piece number $N$ should be larger than the total number of switches in time-optimal control trajectories.}
\end{hyp}
This hypothesis stems from the observation that it is the state or control jumps at the switching times that render a single polynomial incapable of fitting the target trajectory; conversely, if there are no jumps, such as in a bang control, a polynomial is more than sufficient to represent the trajectory. We argue that having a $N$ larger than the total number of switches suffices to yield satisfactory solution quality, as each jump can be accounted for by at least one polynomial. This hypothesis will be verified in Section \ref{subsubsec:poly_num}.

Our goal is to generalize this idea to time-optimal quadrotor flights. We realize that the jumps in collective thrust and body rates shape the dominant structure of optimal trajectories, and their properties can be near-equivalently studied through the non-dimensional model introduced in Section \ref{sec:properties}. Although a two-dimensional model cannot cover all three-dimensional maneuvers, for most time-critical tasks associated with an agile platform such as a quadrotor, there often exists a primary moving direction where the most significant motion occurs (e.g., when the distance is large, the vehicle's motion typically aligns with the direction to the next waypoint), and their characteristic maneuver properties can be analyzed by using the simplified model. Note that the non-dimensional model has two control inputs, and thus the total switches in both inputs must be considered and $N$ should be assigned correspondingly. 

In the next two sections, we elucidate the principle of picking $N$, with and without prior knowledge about the switching number.

\subsection{Inference of $N$ with Known Switch Numbers}

For certain tasks, e.g., purely horizontal translations, previous works \cite{hehn2012performance} have provided a lookup table/graph for the exact switch numbers w.r.t. the flight distance. In this case, we can directly compute $N$ by using the formula below:
\begin{equation}
N=N_{T}+N_{R}+N_{SE}+1,\label{equ_polypiece}
\end{equation}
where $N_{T}$ is the switch number for the collective thrust and $N_{R}$ for the body rates. $N_{SE}=\{0,1,2\}$ indicates whether we have information about their optimal values at the start or end points; it is safe to set $N_{SE}\!=\!2$ if none of this information is accessible. Lastly, we add the last term to ensure that at least one polynomial can be assigned.

\subsection{Inference of $N$ without Exact Switch Numbers}

Considering the complexity of quadrotor maneuvers, it is impossible to obtain $N_{T}$ and $N_{R}$ for arbitrary start and end states. However, if their upper bounds are available, we are guaranteed to find a piecewise polynomial whose $N$ suffices to capture the structure of optimal trajectories. Fortunately, Theorem \ref{the:nontrivial} has answered this question--- $\overline{N_{T}}\!=\!5$ and $\overline{N_{R}}\!=\!4$---which enables the development of a robust AOS scheme, as depicted in Algorithm \ref{alg:robust_aos}, capable of tackling a broad range of time-optimal two-state problems. Note that the robust AOS is a conservative approach as it starts with the largest $N$ and then iteratively decreases it until convergence is achieved. This is because certain values of $N$ do not align well with specific trajectory structures, leading to issues in time allocation that hinder convergence. Nevertheless, thanks to the minimal representation of time-optimal trajectories, our robust AOS still consistently outperforms state-of-the-art approaches in terms of speed while achieving comparable solutions.

\begin{algorithm}[H]
\caption{Robust Automatic Optimal Synthesis}
\begin{algorithmic}[1]
\renewcommand{\algorithmicrequire}{\textbf{Input:}}
\renewcommand{\algorithmicensure}{\textbf{Output:}}
\REQUIRE $\bar{\mathbf{x}}_{0}$, $\bar{\mathbf{x}}_{f}$, $N_{SE}\!=\!\{0,1,2\}$, $(\overline{N_{T}}, \overline{N_{R}})\!=\!(5,4)$
\ENSURE  $t_{f}$, $\mathbf{x}$, $\mathbf{u}$
\STATE $(T_{init},\mathbf{y}_{init})$ $\leftarrow$ SolveAOS($N=1$) for Problem (\ref{equ_top_poly})
\STATE $N=\overline{N_{T}}+\overline{N_{R}}+N_{SE} + 1$
\WHILE {$N > 1$}
\STATE WarmstartAOS($\mathbf{y}_{init}$)
\STATE $(T_{1,...,N},\mathbf{y}_{1,...,N})$ $\leftarrow$ SolveAOS($N$)  for Problem (\ref{equ_top_poly})
\IF {Success}
\STATE $\mathbf{y}$ $\leftarrow$ Concatenate($\mathbf{y}_{1,...,N}$)
\STATE $t_{f}$ $\leftarrow$ $ \sum_{k=1}^{N}T_{k}$
\STATE $\mathbf{x}$ $\leftarrow$ $\boldsymbol{\Psi}_{\mathbf{x}}(\mathbf{y},...,\mathbf{y}^{(s-1)})$
\STATE $\mathbf{u}$ $\leftarrow$ $\boldsymbol{\Psi}_{\mathbf{u}}(\mathbf{y},...,\mathbf{y}^{(s)})$
\RETURN {($t_{f}$, $\mathbf{x}$, $\mathbf{u}$)}
\ENDIF
\STATE $N$ $\leftarrow$ $N-1$
\ENDWHILE
\end{algorithmic} \label{alg:robust_aos}
\end{algorithm}

\section{Results: Two-State Problems}\label{sec:perf_eval}

This section provides a thorough evaluation of AOS in solving the classical two-state problem.

\subsection{Introduction of Comparison Sets and Modelling Accuracy}

We compare our method with six mainstream approaches: PB \cite{hehn2012performance}, CPC \cite{foehn2021time}, TOGT \cite{qin2023time}, RTG \cite{hehn2015real}, FR \cite{han2021fast}, and REF \cite{zhou2023efficient}. Note that (i) PB and RTG are indirect methods, (ii) CPC and REF are direct methods based on MS, and (iii) TOGT and FR are polynomial-based methods. To ensure consistent trajectory quality, we unify the numerical integrator used by CPC and REF to the fourth-order Runge-Kutta method (RK4), and conduct multiple tests until the resulting sampling time is below 0.03 s. Note that these two methods are equivalent in non-waypoint flights, and thus we choose either one to compare in such cases.

There are two common quadrotor models used in previous research, one using collective thrusts and body rates as control inputs, denoted as Model S (simplified), and the other using single-rotor thrusts as control inputs while considering body-rate constraints, denoted as Model R (realistic). If not specified otherwise, the default model would be Model R. The suffix "S" will be added to the method's name if Model S is utilized.

\begin{table}[!htbp]
\centering
\caption{Quadrotor configurations}\label{tab_params}
\tabcolsep=0.08cm	
\begin{threeparttable}
\begin{tabular}{@{}ccccccc@{}}
    \toprule
    & $m$ [kg] & $l$ [m] & $\mathbf{J}_{diag}$ [gm$^2$] & $(\underline{f},\overline{f})$ [N] & $c_{\tau}$ [1]   & $\overline{\boldsymbol{\omega}}$ [rad$\,$s$^{-\!1}$] \\ \midrule
    Quad STD & 1.0  & 0.15  & [5.0,5.0,10]     & [0.25, 5.0]       & 0.01       & [10, 10, 10] \\   
    Quad RPG & 0.85  & 0.15  & [1.0,1.0,1.7]     & [0.1, 6.88]       & 0.05       & [15, 15, 3]             \\
    Quad FGG & 1.0  & 0.08  & [4.9,4.9,6.9]     & [0.1, 9.0]       & 0.136       & [10, 10, 3]             \\
    Quad FSC & 1.005 & 0.125 & [2.5,2.1,4.3] & [0.1, 9.0]     & 0.022      & [10, 10, 3]              \\ \bottomrule
\end{tabular}
\begin{tablenotes}[para,flushleft]
    \item Note that $l$ denotes the quadrotor's arm length, $\mathbf{J}_{diag}$ the diagonal elements of the inertia matrix, and $c_{\tau}$ the torque constant.
\end{tablenotes}
\end{threeparttable}
\end{table}

\subsection{Time Optimality}\label{subsec:h2h}

\begin{figure}[!htbp]
\captionsetup[subfloat]{farskip=2pt,captionskip=1pt}
\centering
\subfloat[]{\includegraphics[width=0.48\textwidth]{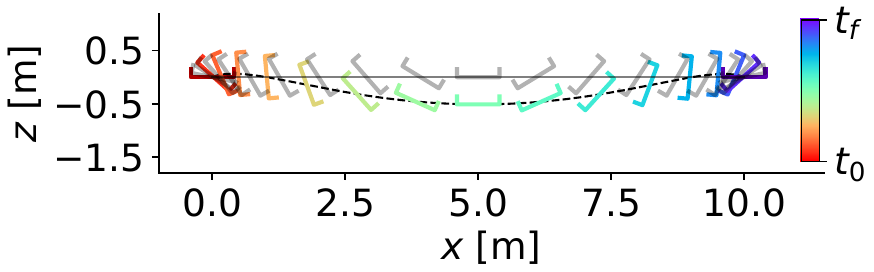}}
\hfill
\subfloat[]{\includegraphics[width=0.48\textwidth]{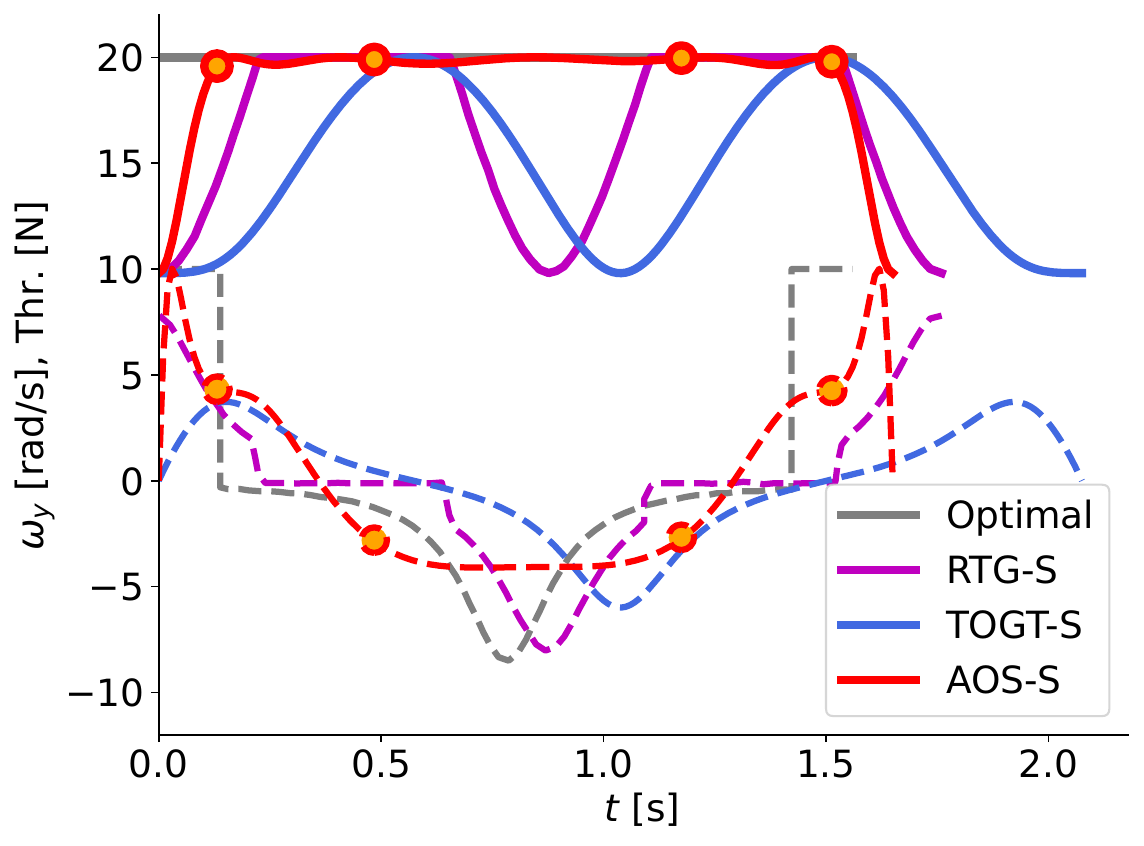}}
\caption{(a) Comparison of time-optimal maneuvers for a purely horizontal translation of 10 m. The dashed line (rainbow snapshots; total duration 1.65 s) represents the trajectory obtained from our AOS-S, while the solid line (grey snapshots; total duration 2.07 s) represents the trajectory obtained from TOGT-S \cite{qin2023time}. Snapshots are plotted at an interval of 0.1 s. The trajectory from RTG-S (total duration 1.76 s) \cite{hehn2015real} is skipped as it is similar to TOGT-S. Note that our algorithm results in an aggressive maneuver characterized by significant altitude changes. (b) Corresponding control inputs. The grey lines are the true time-optimal translation \cite{hehn2012performance} (total duration 1.56 s). The solid lines represent the thrust input while the dashed lines denote the rotational rate input. We observe that our approach (red) successfully captures the key features of the optimal trajectory, whereas TOGT-S (blue) and RTG-S (purple) have significant deviations.}\label{fig_hover_to_hover}
\end{figure}  

\subsubsection{Horizontal Hover-to-Hover Flight}\label{subsubsec:horizontal_h2h}
We compare the performance in purely horizontal translations of 3, 6, 9, 12, and 15 m. The switch numbers for this task are accessible in \cite{hehn2012performance}: $N_{R}\! = \!2$ and $N_{T}\! = \!0$. Therefore, using Eq. (\ref{equ_polypiece}), we get the minimal $N\!=\!5$ for our approach. The model parameter can refer to the Quad STD configuration in Table \ref{tab_params}. 

\begin{table}[htbp]
	\tabcolsep=0.11cm
	\centering
	\caption{Comparison of Timing in Purely Horizontal Translations}\label{tab_ph_timing}
	\begin{tabular}{@{}ccccccc@{}}
		\toprule
		$x$ vs $t_{f}$[s] & PB-S \cite{hehn2012performance}  & CPC-S  \cite{foehn2021time} & AOS-S  &  TOGT \cite{qin2023time} &  CPC \cite{foehn2021time}  & AOS \\ \midrule
		3 m     & 0.898 & 0.891   & 0.956   & 1.171 & 0.918 & 1.084      \\
		6 m     & 1.231 & 1.227   & 1.307   & 1.610 & 1.255 & 1.398      \\
		9 m     & 1.488 & 1.484   & 1.573   & 1.965 & 1.517 & 1.657      \\
		12 m    & 1.705 & 1.702   & 1.797   & 2.265 & 1.736 & 1.878      \\
		15 m    & 1.895 & 1.894   & 1.994   & 2.530 & 1.933 & 2.075      \\
		\bottomrule
	\end{tabular}
\end{table}

Table \ref{tab_ph_timing} benchmarks AOS against three approaches under different model fidelity. It shows that with Model S, the resulting timing of our method is 5.8\% slower than that of CPC-S and 6.5\% slower than that of PB-S on average. When taking into account the rotation dynamics and single-rotor limits, AOS is 10.8\% slower than CPC. This slightly longer duration is expected because discretization-based methods have a stronger capacity to simulate sudden state or control changes compared to polynomial representations. But when compared to the same method type such as TOGT, which is 29.3\% slower, the AOS still stands out regarding solution quality. Fig. \ref{fig_hover_to_hover} displays the obtained input trajectories in a horizontal displacement of 10 m. It can be observed that AOS-S allocates appropriate durations for all five polynomials to maintain maximal thrust input, reaching an optimality gap of 5.8\%. TOGT-S only reaches the maximal thrust twice and never hits the maximal rotational rate. RTG-S faces a similar issue of not being able to keep a maximal thrust input when the vehicle rotates. As a consequence, these two approaches exhibit noticeable deviations from the optimal solution. 

\subsubsection{Diagonal Hover-to-Hover Flight}\label{subsubsec:diagonal_h2h}

We evaluate our approach in a diagonal flight with 8 m of horizontal translation and 8 m of vertical translation, and the resulting trajectories are demonstrated in Fig. \ref{fig_task_02}. Note that the switch number for this task is also available: $N_{R}\! = \!4$ and $N_{T}\! = \!1$. This leads to the minimal $N\!=\!8$ for AOS. To make most approaches applicable, we mainly consider Model S in the following comparison. 

\begin{figure}[!htbp]
\centering
\hspace{-0.02\textwidth}
\includegraphics[width=0.48\textwidth]{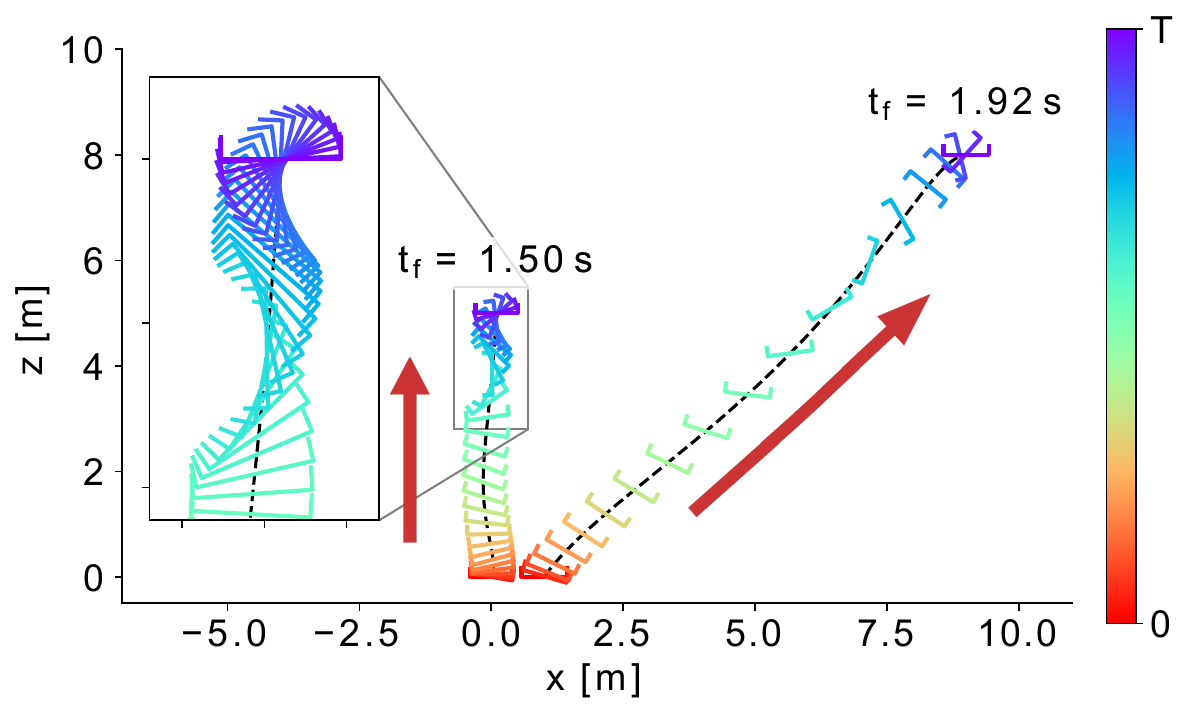}
\caption{Time-optimal maneuvers produced by AOS-S for a purely vertical translation of 5 m (left) and a diagonal translation of 8 m horizontally and 8 m vertically (right). Note that in both scenarios, the vehicle reaches a pitch angle of approx 140$^{\circ}$ in the final deceleration phase.}\label{fig_task_02}
\end{figure}

The results show that CPC-S converges to the same result of 1.76 s as PB-S, achieving true time optimality as anticipated. The durations of RTG-S and TOGT-S are 2.65 s and 2.53 s, respectively, which is substantially longer than the global optimum, and both methods exhibit worse performance due to the increased complexity of the optimal control policy compared to the previous experiment. By comparison, the proposed algorithm is able to maintain a consistently low optimality gap of 7.9\% (total duration 1.92 s). 

\subsubsection{Vertical Hover-to-Hover Flight}\label{subsubsec:vertical_h2h}

We further evaluate their performance in a purely vertical translation of 5 m, a task that involves a more complex optimal control policy---a 360$^{\circ}$ flip. The optimal trajectory obtained from our AOS-S is plotted in Fig. \ref{fig_task_02}. We first note that the resulting timing is 1.5 s, which is 15.3\% slower than 1.3 s from PB-S. The optimality gap increases because it is too challenging for polynomials to produce a flip maneuver that entails rapid and frequent value changes in a very short period. Nevertheless, the attempt of AOS-S to generate a flip is quite evident, which surpasses prior polynomial-based methods such as TOGT and FR that show no signs of conducting a flip.

\subsubsection{Flights with Random Start \& End States}

\begin{figure}[!htbp]
\centering
\includegraphics[width=0.48\textwidth]{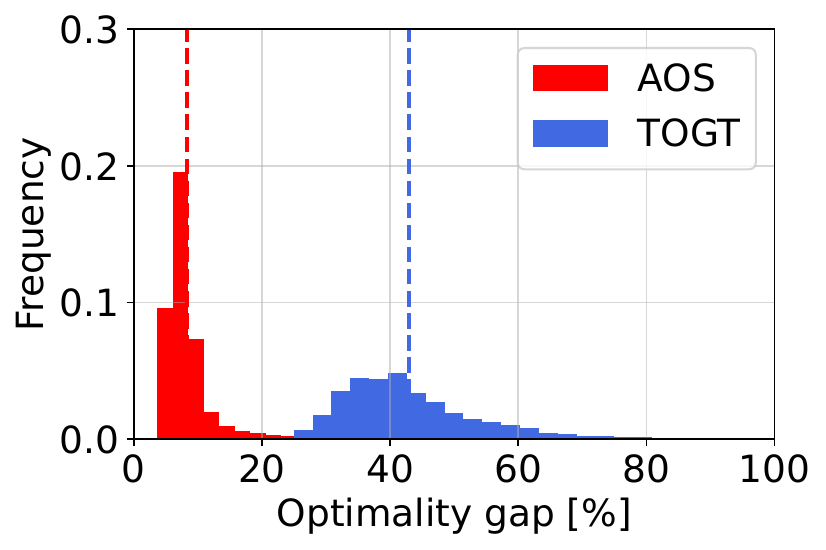}
\caption{Comparison of the optimality gap of AOS (red, medium gap 7.2\%) and TOGT (blue, medium gap 40.7\%) under random start and end conditions. The optimal timing is provided by \cite{foehn2021time}. Note that AOS exhibits a more concentrated distribution within a significantly lower error range compared to TOGT.}\label{fig_histogram}
\end{figure}

We assess the time optimality of the robust AOS approach detailed in Algorithm \ref{alg:robust_aos} under random start and end states---situations where the exact number of switches is inaccessible.

We create a grid in a rectangular cuboid within a volume of $20\!\times\!20\!\times\!10$ m$^3$ centered at the origin. The distance between two cells is $5$ m. The start position is set at the origin. For every end position, we conduct 100 Monte Carlo experiments, uniformly sampling the roll and pitch angles in a range of $(0,\pi/2)$ as well as velocity in a range of 0 to 10 m/s at each direction. We use the results from CPC as ground truth in order to quantify optimality gap.

\begin{figure}[!htbp]
\centering
\includegraphics[width=0.48\textwidth]{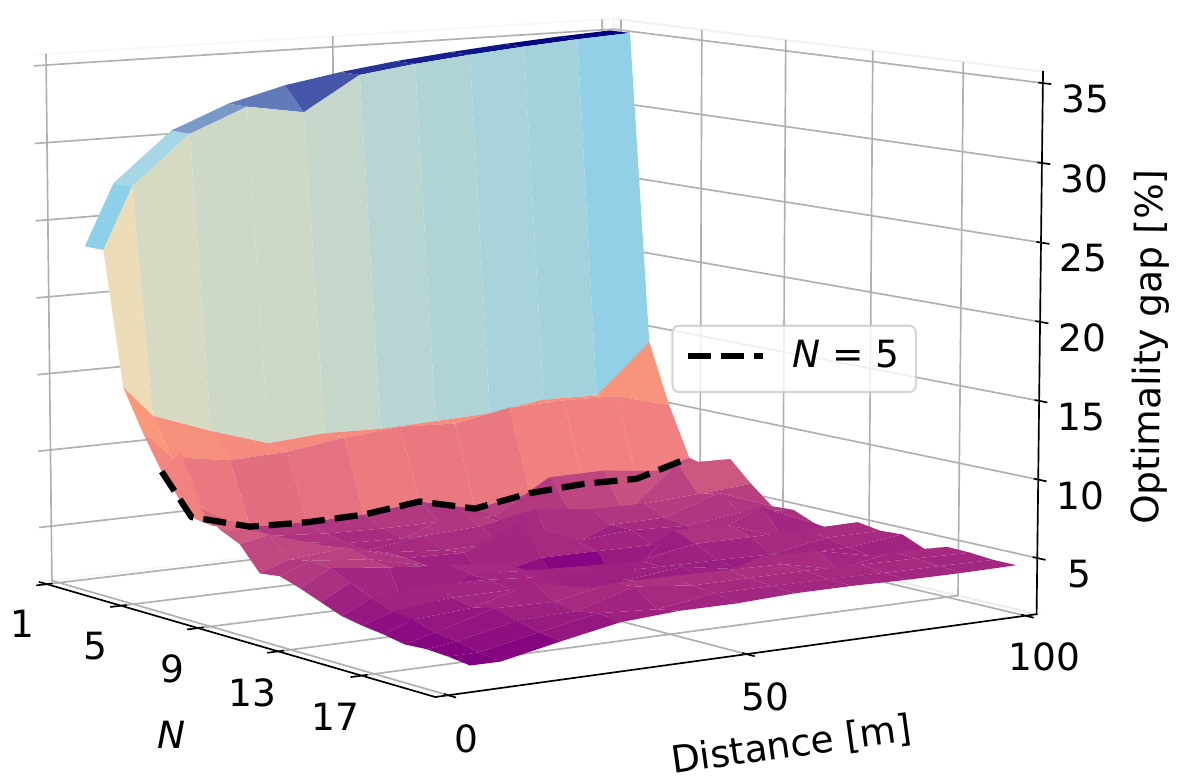}
\caption{Optimality gap of AOS in purely horizontal translations with polynomial number $N$ ranging from 1 to 20. Note that the most substantial drop in gap occurs when $N\!=\!5$, which aligns with the minimal value obtained from Eq. (\ref{equ_polypiece}). Moreover, further increasing $N$ only results in marginal improvements.}\label{fig_error_mesh}
\end{figure}

Fig. \ref{fig_histogram} showcases the distributions of the optimality gap obtained from AOS and TOGT. We find that 88\% of our method's timings fall within a gap between 3.6\% and 8.4\% and have a medium gap of 7.2\%, while the distribution of TOGT is more spread out and has a larger median gap up to 40.7\%. This evidence effectively confirms the consistency of our approach in terms of solution quality.

\subsubsection{Impact of Polynomial Number $N$}\label{subsubsec:poly_num}

This part verifies the efficacy of the proposed formula Eq (\ref{equ_polypiece}) in finding an appropriate $N$. Let's first consider purely horizontal translations with a setup similar to Section \ref{subsubsec:horizontal_h2h}. We adopt Model R and use the results from CPC as ground truth. We apply different values of $N$ to AOS for tasks with target distances ranging from 1 to 100 m. The gap to the optimal duration is recorded and displayed in Fig. \ref{fig_error_mesh}. To begin with, We notice that the solution accuracy of AOS is consistent for all distances with the same $N$. Secondly, the optimality gap quickly falls to approx. 6\% when increasing $N$ from 1 to 5, and it plateaus thereafter. This result demonstrates that $N\!=\!5$ is the sweet spot for this specific problem, as it achieves nearly the best possible solution accuracy with a minimum number of polynomials. Note that Eq (\ref{equ_polypiece}) also suggests $N\!=\!5$, which verifies its effectiveness.

In the second test, we consider diagonal translations with a setup similar to Section \ref{subsubsec:diagonal_h2h}. As shown in Fig. \ref{fig_diag_error_mesh_N}, we see that Eq (\ref{equ_polypiece}) again yields the optimal value of $N\!=\!8$ that balances time optimality and the size of the decision variable. These two numerical studies suffice to validate Hypothesis \ref{hyp:optimal_n} as well as the effectiveness of the proposed formula in choosing $N$.

One last observation is that introducing moderately redundant polynomials (e.g., setting $N\!=\!12$) will not increase the optimality gap. This feature is the foundation for the robust AOS approach, as it indicates that having redundant polynomials is safe for solution accuracy.

\begin{figure}[!t]
\centering
\includegraphics[width=0.48\textwidth]{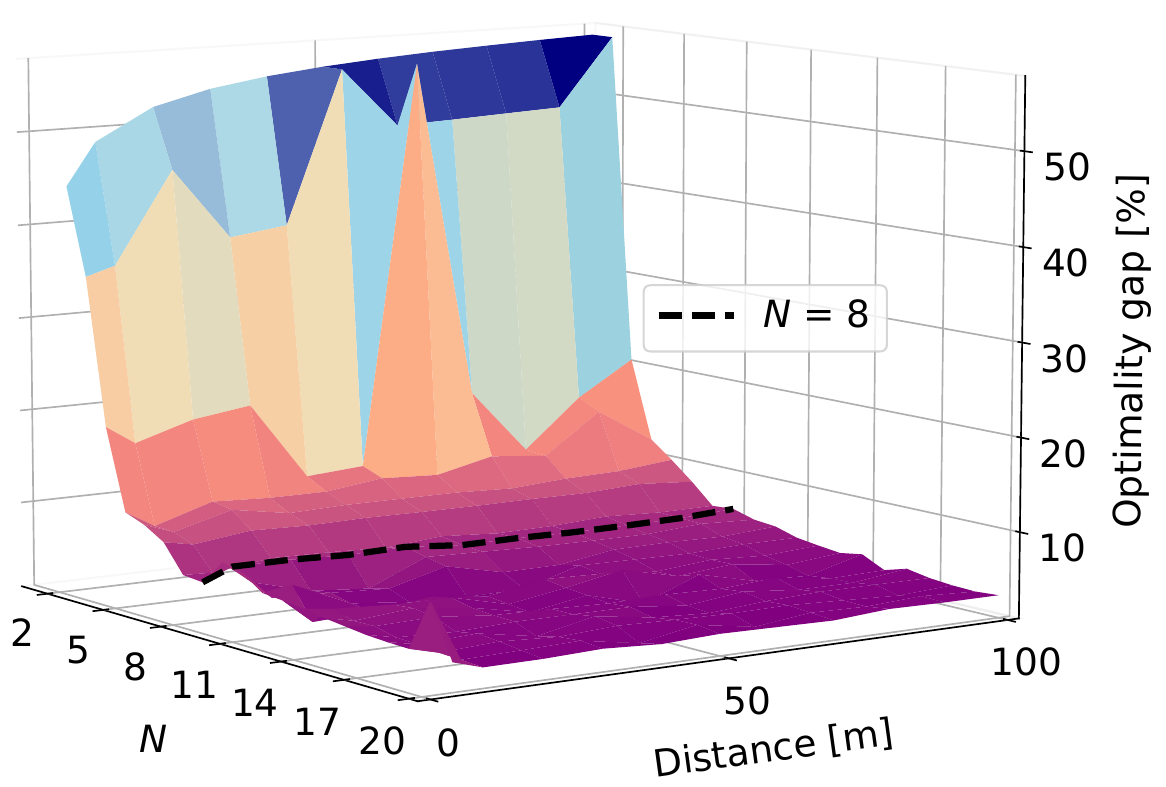}
\caption{Optimality gap of AOS in diagonal translations with polynomial number $N$ ranging from 1 to 20. The displacements in the $x$ and $y$ axes are identical. We observe that the most substantial drop in gap happens when $N\!=\!8$, which also aligns with the minimal value obtained from Eq. (\ref{equ_polypiece}).}\label{fig_diag_error_mesh_N}
\end{figure}

\subsection{Evaluation on Long-Rang Flights}

We evaluate the performance of AOS, CPC, and TOGT regarding time optimality and computational cost in flights with varying target distances. 

\begin{figure}[!htbp]
\centering
\captionsetup[subfloat]{farskip=2pt,captionskip=1pt}
\centering
\subfloat[]{\includegraphics[width=0.48\textwidth]{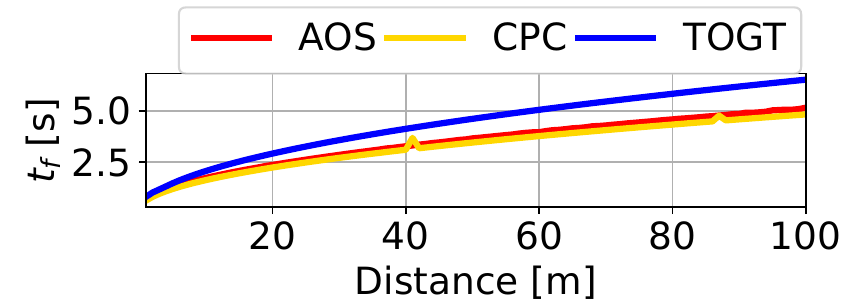}\label{fig_ts_comp}}
\hfill
\subfloat[]{\includegraphics[width=0.48\textwidth]{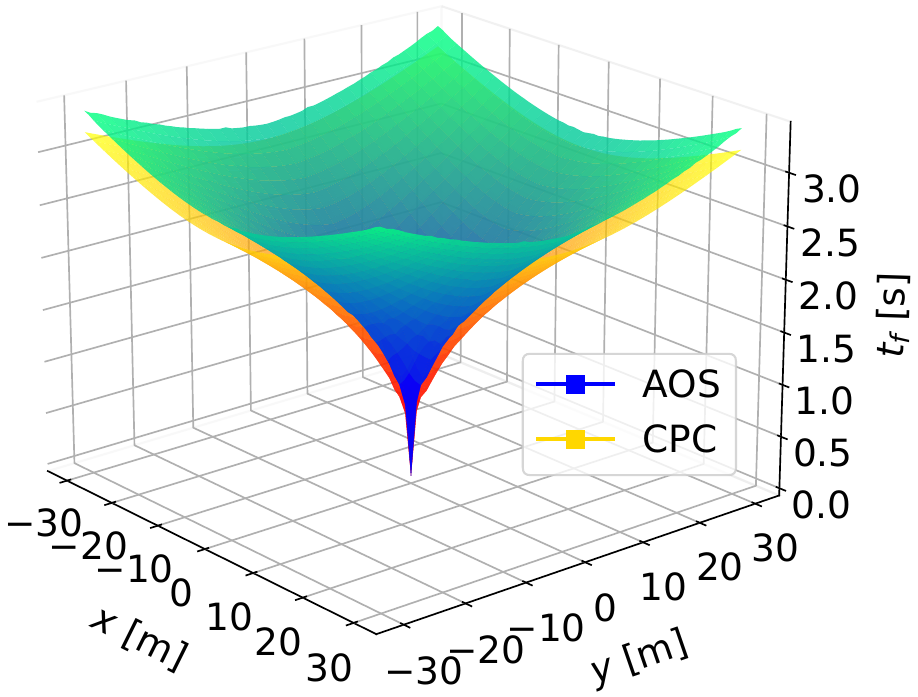}\label{fig_dual_mesh}}
\caption{(a) Time-optimal duration from AOS, CPC and TOGT in purely horizontal translations with flight ranges increasing from 1 to 100 m. (b) Time-optimal duration from AOS and CPC in purely horizontal translations spanning the $xy$ plane. Note that AOS consistently achieves comparable solution accuracy to CPC while TOGT showcases a noticeable gap.}\label{fig:long_range_flight}
\end{figure}

The results are shown in Fig. \ref{fig:long_range_flight}. Specifically, Fig. \ref{fig_ts_comp} demonstrates that the trajectory duration obtained from AOS aligns closely with that from CPC for all distances tested (average gap 6.1\%), while Fig. \ref{fig_dual_mesh} shows that AOS generalizes well to 3D, maintaining consistent time optimality across all flight directions.


\begin{figure}[!htbp]
\centering
\hspace{-0.03\textwidth}
\includegraphics[width=0.42\textwidth]{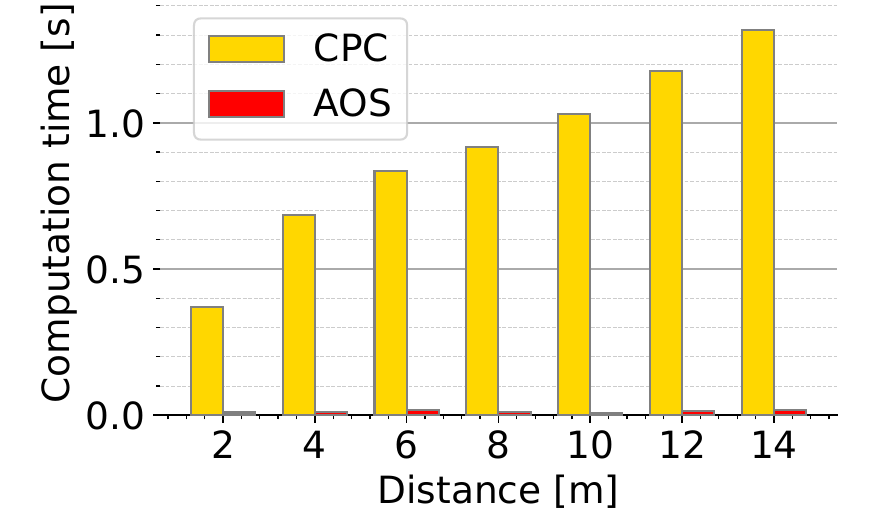}
\caption{Comparison of computation times of AOS and CPC in flights with varying distances. Note that AOS is notably faster than CPC, demonstrating almost no increase in time consumption as the mission range climbs.} \label{fig_tc_hist}
\end{figure}

The computation speed of AOS and CPC is illustrated in Fig. \ref{fig_tc_hist}. We observe that as the target distance increases, CPC requires longer computation times to maintain a sampling time within 0.03 seconds. Particularly, it takes more than 1 s to process the case of 10 m. In contrast, AOS can maintain the same $N$ across similar tasks without increasing the optimality gap, and therefore achieve a relatively stable computation time of 13 ms on average, which is significantly lower than that of CPC. By the way, AOS only takes 66 ms to process the case of 100 m while CPC spends more than an hour, with their resulting durations differing by only 6.7\%. This evidence demonstrates the remarkable time efficiency of the proposed approach.




\subsection{Experimental Validation}\label{subsec:experiment}

\begin{figure}[!htbp]
\centering
\includegraphics[width=0.45\textwidth]{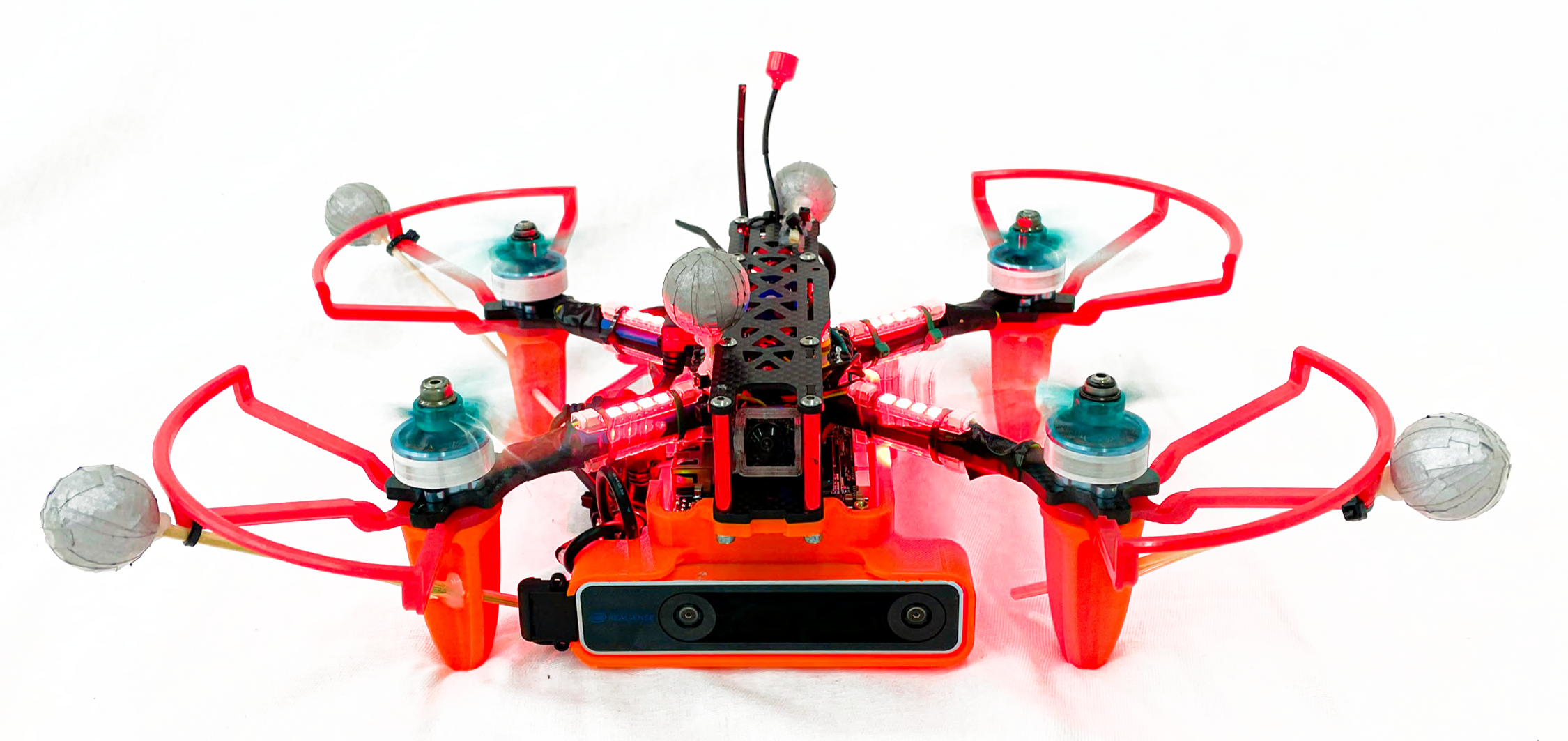}
\caption{Platform for experimental validation equipped with a Jetson TX2 NX computer, a Betaflight flight control unit (FCU), and infrared-reflective markers for	motion capture.}
\label{fig_exp_drone}
\end{figure}

Our flight arena is equipped with a motion capture system with a tracking volume of $6\!\times\!6\!\times\!2.5$ m$^3$. A laptop with an Intel Core i7 processor is used for trajectory generation. Fig. \ref{fig_exp_drone} details the configuration of the physical platform. The onboard computer runs the autopilot \cite{foehn2022agilicious} and a nonlinear model predictive controller (NMPC) \cite{verschueren2022acados} for trajectory tracking. Note that the FCU accepts body rates and collective thrust as the lowest-level control inputs. Therefore, we extract the body rates from the NMPC states and compute the collective thrust by summing up thrusts at each rotor. To maintain controllability amidst disturbances, the trajectory is generated with a slightly lower thrust bound than the platform's maximum capacity. See Quad FSC in Table \ref{tab_params} for detailed model parameters.

We execute two time-optimal trajectories produced by AOS to validate its solution quality: one with a diagonal translation of 5 m horizontally and 1.5 m vertically, and the other with a purely horizontal translation of 5 m. The results are visualized in Fig. \ref{fig_exp_diag_hori}. We observe that the quadrotor can precisely follow the planned trajectories, reaching a maximum roll angle of 100.9$^{\circ}$ and a peak velocity of 8.72 m/s in the diagonal maneuver. It shows that our trajectories have excellent dynamic feasibility to facilitate the tracking.


\begin{figure}[!htbp]
	\centering
\includegraphics[width=0.45\textwidth]{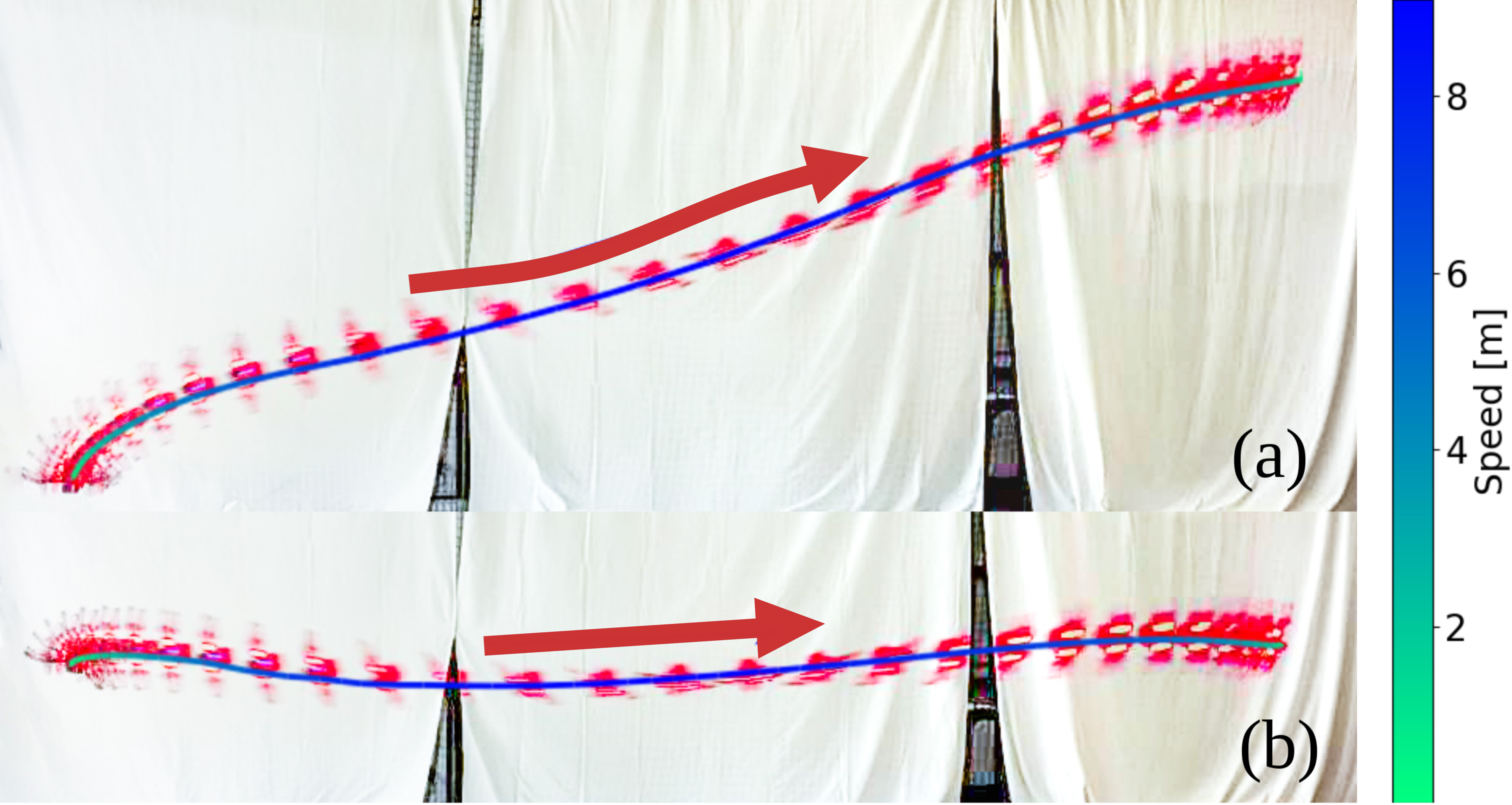}
\caption{Time-optimal maneuvers obtained by executing our planned trajectories (blue lines) with a real-world platform for (a) a diagonal translation of 5 m horizontally and 1.5 m vertically (maximum roll 100.9$^{\circ}$; maximum velocity 8.72 m/s), and (b) a pure horizontal displacement of 5 m (maximum roll 87.3$^{\circ}$; maximum velocity 8.39 m/s).}\label{fig_exp_diag_hori}
\end{figure}

\begin{figure*}[!htbp]
\centering
\includegraphics[width=0.567\textwidth]{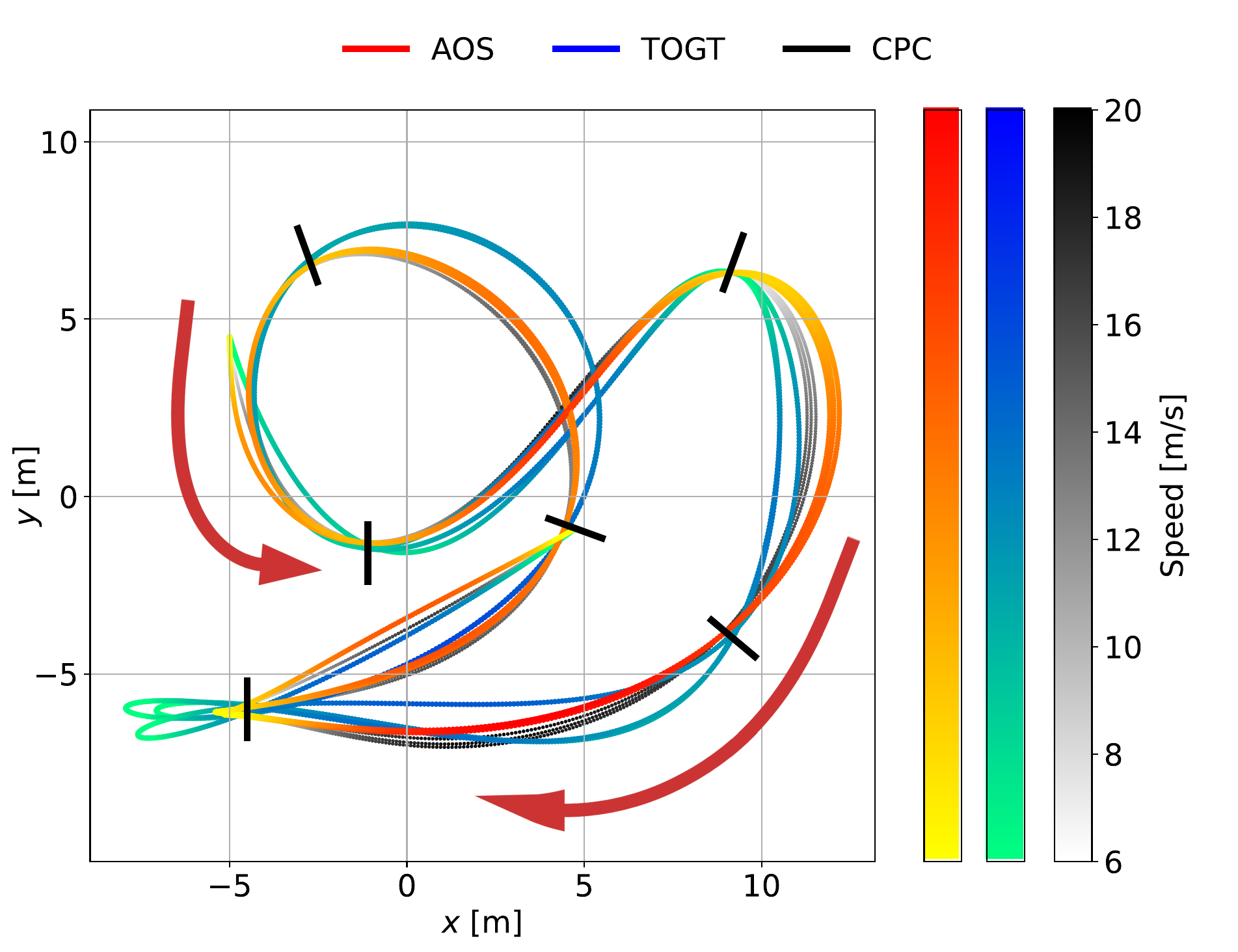}
\includegraphics[width=0.31\textwidth]{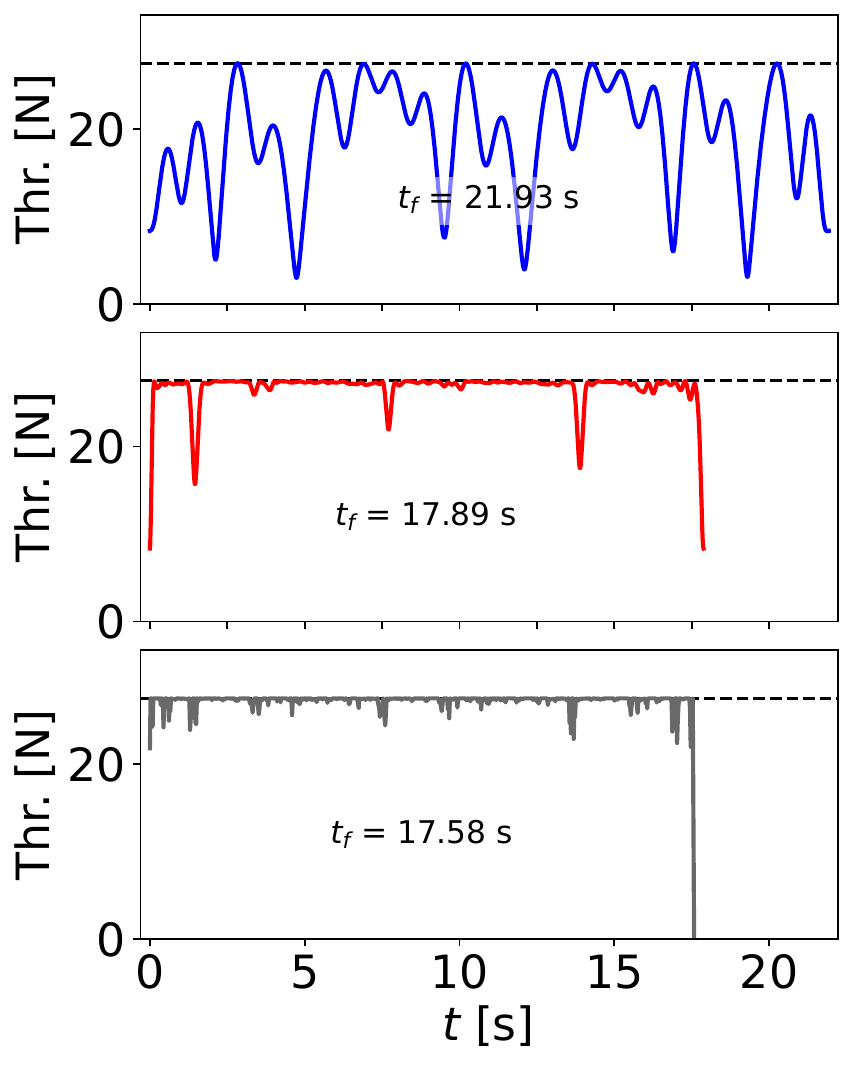}
\caption{(a) Time-optimal trajectories planned for the Split-S track \cite{song2023reaching}. This mission includes 19 gates located at 7 separate locations. Note that the gate in the bottom left corner comprises two vertically stacked gates. (b) Comparison of the collective thrust inputs from TOGT, AOS, and CPC. Note that AOS can maintain maximal collective thrusts constantly and lead to merely 1.7\% longer duration than CPC.}\label{fig_uzh_19wp}
\end{figure*}

\section{Application to Time-Optimal Waypoint Flights}

Our approach can readily incorporate waypoint constraints and demonstrate time-optimal waypoint flights. This section introduces the technical details in this regard.

\subsection{Problem Formulation}

Consider a task with $L\!-\!1$ waypoints. We denote the waypoint position as $\bar{\mathbf{p}}_i$ and the corresponding traversal time as $t_{i}$, where $i=1,...,L\!-\!1$. Given the structure of multi-waypoint flights, we can divide the original optimization problem into $L$ sub-problems based on the waypoint order, leading to a $L$-stage optimal control problem described below:
\begin{subequations}\label{equ_towpf}
\begin{align}\min_{\mathbf{x}(\cdot),\mathbf{u}(\cdot),t_{1,...,L}}\quad & \sum_{i=1}^{L}\int_{t_{i-1}}^{t_{i}}1dt\\
\text{s.t.}\quad & \mathbf{x}(t_{0})=\bar{\mathbf{x}}_{0},\;\mathbf{x}(t_{L})=\bar{\mathbf{x}}_{f}\\
& \dot{\mathbf{x}}_{i}(t)\!=\!\mathbf{f}(\mathbf{x}_{i}(t),\mathbf{u}_{i}(t)),t\!\in\![t_{k-1},t_{k}],,\\
& \mathbf{x}_{i}(t_{i})=\mathbf{x}_{i+1}(t_{i}),i=1,...,L-1,\\
& \mathbf{h}_{i}(\mathbf{x}_{i}(t),\mathbf{u}_{i}(t))\!\leq\!\mathbf{0},\;t\!\in\![t_{i\!-\!1},\!t_{i}],\\
& \mathbf{p}^{w}(t_{i})-\mathbf{\bar{p}}_{i}=\mathbf{0},i=1,...,L-1,\\
& t_{i-1}-t_{i}<0,
\end{align}
\end{subequations}
where $t_{L}\!=\!t_{f}$ is the end time. Note that $\mathbf{x}_{i}(t_{i})$ is the quadrotor's state when passing the $i$-th waypoint, and therefore a waypoint constraint should be imposed on its position component, $\mathbf{p}^{w}(t_{i})$ as shown in Eq. (\ref{equ_towpf}f). If a certain traversal tolerance is allowed for these waypoints, it can be formulated as ball constraints, as discussed in \cite{qin2023time}.


It can be seen that for each sub-problem, there is a start state, $\mathbf{x}_{i}(t_{i})$, and an end state, $\mathbf{x}_{i}(t_{i+1})$, and there are no geometrical constraints in between. Therefore, each sub-problem can be interpreted as a two-state problem discussed in Section \ref{sec:aos}, and the maximum switch number of the thrust and rotational rate inputs derived in Theorem \ref{the:nontrivial} can apply to each sub-trajectory. Using this principle, we can formulate AOS for each sub-problem and optimize all sub-trajectories along with the intermediate states in a single optimization problem. Introducing the polynomial representation, we can acquire the following $L\times N$-stage problem:
\begin{subequations}\label{equ_towpf_poly}
\begin{align}\min_{\mathbf{y}_{1\!,...,\!L,1\!,...,\!N}}\quad & \sum_{i=1}^{L}\sum_{j=1}^{N}T_{i,j}\\
	\text{s.t.}\quad & \mathbf{y}_{1,1}^{[s-1]}(0)=\bar{\mathbf{y}}_{0}^{[s-1]},\\
	& \mathbf{y}_{L,N}^{[s-1]}(T_{N})=\bar{\mathbf{y}}_{f}^{[s-1]},\\
	& \mathbf{y}_{i,j}^{[s\!-\!1]}(0)=\mathbf{y}_{i,j\!-\!1}^{[s\!-\!1]}(T_{j}),\;j\!=\!2,...,N,\\
	& \mathbf{h}_{\boldsymbol{\Psi}}(\mathbf{y}_{i,j})\leq\mathbf{0},\;-T_{i,j}<0,\\
	& \mathbf{y}_{i,N}(T_{i,N})=\mathbf{\bar{p}}_{i},\;i\!=\!1,...,L\!-\!1,
	\end{align}
\end{subequations}
where $\mathbf{y}_{i,j}$ represents the $j$-th polynomial piece of the $i$-th segments with $T_{i,j}$ its duration. Since the problem structure as well as the polynomial family are unchanged, Problem (\ref{equ_towpf_poly}) can be tackled in a similar manner. Note that one can also integrate gate constraints as illustrated in reference \cite{qin2023time}. Readers can find the implementation details in our open-sourced codes.

\begin{table*}[!htbp]
\centering
\tabcolsep=0.12cm
\caption{Comparison of Time-Optimal Planners in Waypoint Flights}\label{tab_wpf_comp}
\begin{threeparttable}

\begin{tabular}{cccccccccccccccccc}
\toprule
\multirow{2}{*}{Waypoint No.} & \multicolumn{2}{c}{CPC \cite{foehn2021time}} &  & \multicolumn{2}{c}{REF \cite{zhou2023efficient}} &  & \multicolumn{2}{c}{Sampling-based \cite{penicka2022minimum}} &  & \multicolumn{2}{c}{FR \cite{han2021fast}} &  & \multicolumn{2}{c}{TOGT \cite{qin2023time}} &  & \multicolumn{2}{c}{\textbf{AOS (ours)}} \\ \cline{2-3} \cline{5-6} \cline{8-9} \cline{11-12} \cline{14-15} \cline{17-18}
& c.t. [s]      & T [s]     &  & c.t. [s]      & T [s]     &  & c.t. [s]       & T [s]       &  & c.t. [s]        & T [s]        &  & c.t. [s]        & T [s]        &  & c.t. [s]  & T [s] \\ \toprule
19                            & 2718             & 17.58        &  & 61.32            & 17.58        &  & 355$^{*}$                & 18.80             &  & 0.016              & 37.95           &  & 1.50               & 21.93           &  & 6.240         & 17.89    \\
33                            & 6438             & 29.92        &  & 110.5            & 29.83        &  & 213$^{*}$                & 32.00             &  & 0.026              & 63.14           &  & 2.86               & 36.71           &  & 12.28        & 30.29    \\
47                            & 28405           & 44.01        &  & 150.4            & 42.09        &  & 287$^{*}$               & 45.21             &  & 0.024              & 88.33           &  & 3.49               & 51.51           &  & 17.20        & 42.68    \\
61                            & --               & --           &  & 201.9            & 54.35        &  & 661                & 58.42             &  & 0.041              & 113.5           &  & 5.18               & 66.30           &  & 24.81        & 55.02    \\
75                            & --               & --           &  & 255.7            & 66.60        &  & 272$^{*}$                & 71.62             &  & 0.043              & 138.7           &  & 5.61               & 81.08           &  & 34.67        & 67.44    \\ \bottomrule
\end{tabular}
\begin{tablenotes}[para,flushleft]
\item Note that the superscript $^{*}$ indicates that the algorithm does not fully converge and -- represents optimization failure.
\end{tablenotes}
\end{threeparttable}

\end{table*}

\section{Results: Multi-Waypoint Flight Problems}
\subsection{Time Optimality}
We benchmark our algorithm against the state-of-the-art methods on the Split-S race track \cite{song2023reaching} which consists of 7 gates. The layout can be found in Fig. \ref{fig_uzh_19wp}. The gate center is treated as a waypoint, and a tolerance of 0.3 m is assigned. The mission will start from position $[-5.0,4.5]$ m. After completing a certain number of laps, the vehicle should come to a stop at position $[4.75,-0.9,1.2]$ m. The Quad RPG configuration in Table \ref{tab_params} gives the quadrotor parameters used in this task.

The resulting timings are available in Table \ref{tab_wpf_comp}. As expected, MS methods (CPC and REF) yield the shortest duration. However, as more laps are involved, REF exhibits a monotonically increase in processing time, and CPC fails to converge in the case of 47 waypoints. The sampling-based method \cite{penicka2022minimum} can output results with slightly longer durations after multiple trials, but there is no guarantee of convergence or time optimality. Prior polynomial-based methods (FR and TOGT) showcase fast computation speed but at the cost of suboptimal timings; In the case of 75 waypoints, TOGT is 21.7\% slower than REF while FR directly doubles the lap time. By comparison, our AOS successfully achieves a good trade-off between computation speed and solution accuracy. Firstly, it outperforms sampling-based and polynomial-based methods in terms of the optimality gap. In the 19-waypoint case, its trajectory duration is only 1.7\% longer than CPC and REF, while in the 75-waypoint case, the gap further drops to 1.2\%. Fig. \ref{fig_uzh_19wp} clearly shows that AOS has largely overcome the limitations of polynomials mentioned in \cite{foehn2021time}, showing the capability of maintaining maximal thrust inputs constantly. Furthermore, the required computation time is lower than REF by an order of magnitude, not to mention CPC. Additionally, the result also manifests the AOS's strong scalability w.r.t. the number of waypoints.

\begin{figure}[!htbp]
\centering
\includegraphics[width=0.45\textwidth]{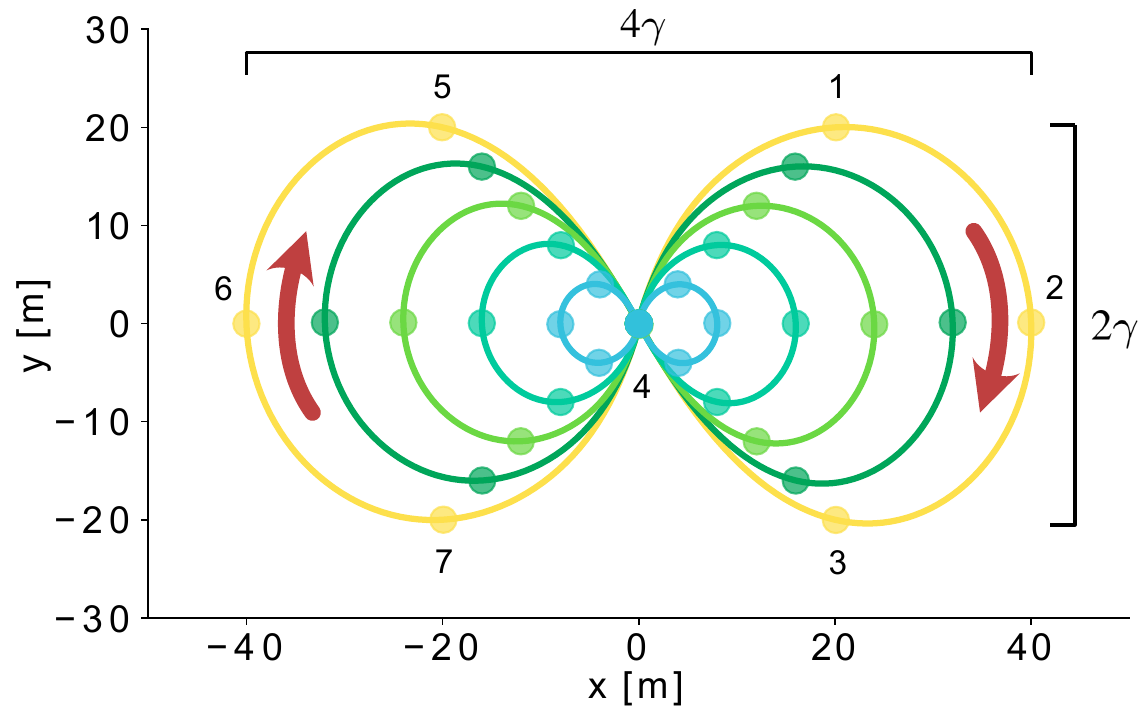}
\includegraphics[width=0.45\textwidth]{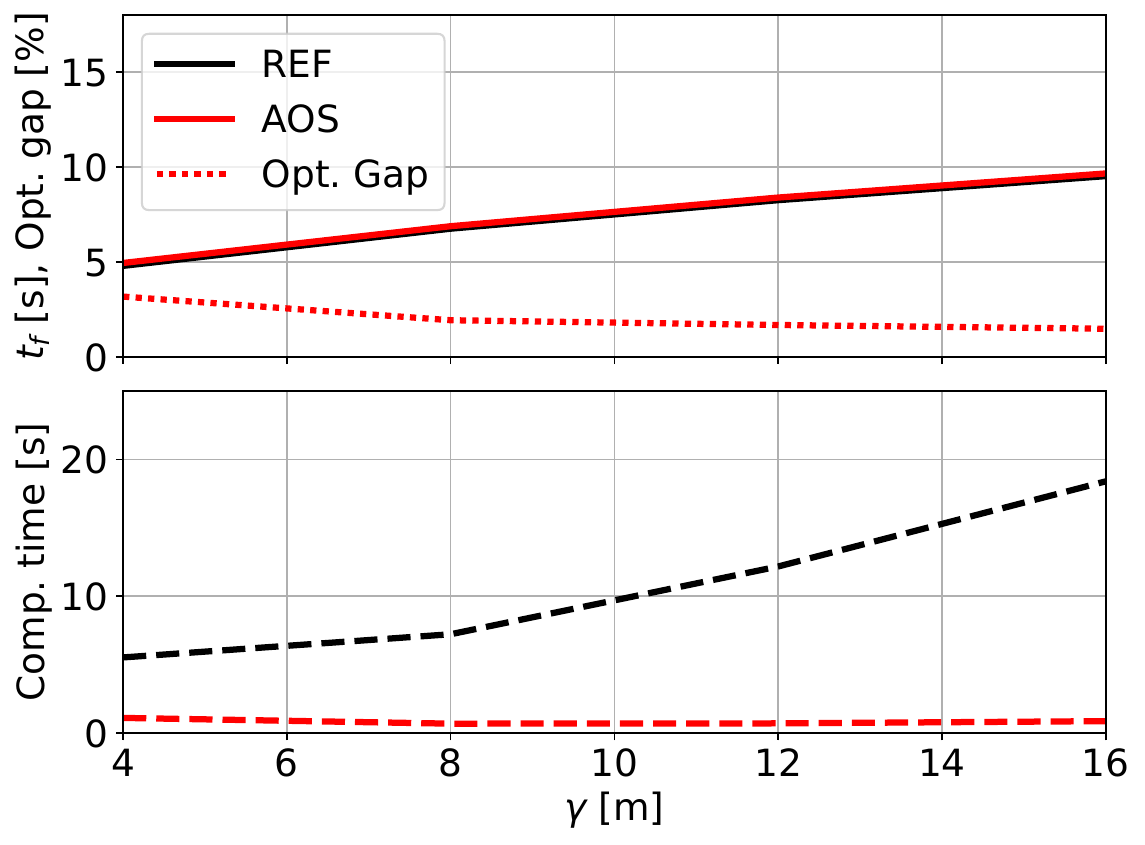}
\caption{Evolution of the optimality gap and computation time as the radius $\gamma$ rises from 4 m to 16 m. The largest layout (yellow) covers an area of $64\times32$ m$^2$. Circles with the same color denote the waypoints in the same round. Note that AOS can maintain a small discrepancy to REF regardless of the size of the layout (2.0\% on average and 1.5\% at $\gamma\!=\!16$). And its computation time is significantly lower than that of REF.}\label{fig:figure8_all}
\end{figure}

\subsection{Evaluation on Long-Range Flights}



\begin{figure*}[!htbp]
\centering
\captionsetup[subfloat]{farskip=2pt,captionskip=1pt}
\centering
\includegraphics[width=0.98\textwidth]
{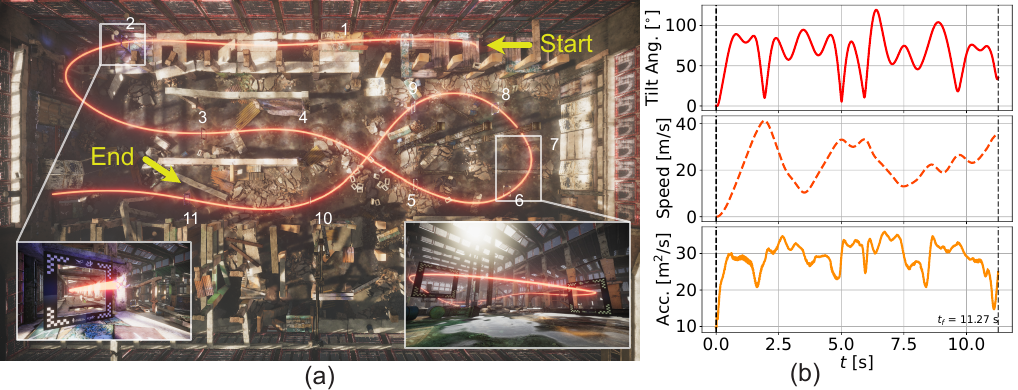}
\caption{(a) Time-optimal trajectory generated by our AOS for the AlphaPilot simulation challenge \cite{foehn2022alphapilot} and executed by NMPC in the FlightGoggle simulator \cite{guerra2019flightgoggles}. This track includes 11 gates with an order indicated by the number. Note that we achieve a lap time of 11.27 s which is faster than the best records of 18 s. (b) Time evaluation of the tilt angle, speed, and acceleration. Our approach showcases superhuman performance in long-range time-critical tasks, reaching a maximum velocity of 148 km/h.}\label{fig:alphapilot_photo}
\end{figure*}

We now quantitatively assess the scalability of the proposed algorithm w.r.t. the waypoint layout. As depicted in Fig. \ref{fig:figure8_all}, we aim to generate a figure-8 path using 7 waypoints, whose layout is detailed in Table \ref{tab:figure8}. Note that the start and end positions are both at the origin. The strategy is to increase the gap between waypoints by controlling the scalar $\gamma$ and we observe how the tested algorithm behaves in terms of computation time and solution accuracy. The results in Fig. \ref{fig:figure8_all} show that AOS has no apparent increase in computation time compared to REF when $\gamma$ climbs from 4 to 16. Meanwhile, it consistently achieves a gap of approx. 2\% to the solution from REF. This sufficiently proves that AOS also keeps its range-insensitive feature in waypoint flights.

To verify its performance in actual long-range time-optimal tasks, we generate racing trajectories using our method for the autonomous drone racing competition conducted in the FlightGoggle simulator \cite{guerra2019flightgoggles}, resulting in a trajectory with a length of 256.8 m within 5 s. Then, the trajectory is executed by an NMPC using the same setup in Section \ref{subsec:experiment}. The quadrotor parameters can refer to Quad FGG in Table \ref{tab_params}. The results are given in Fig. \ref{fig:alphapilot_photo}. It shows that AOS can complete a lap within 11.27 s with a peak velocity of 41.23 m/s, significantly outperforming the best record of approx. 18 s at the simulation competition \cite{mit2024alphapilot}. Please refer to our attached video for the first-person-view footage.


\begin{table}[!htbp]
\begin{center}
\centering
\caption{Waypoint Layout of the Figure-8 Trajectory}\label{tab:figure8}
\begin{tabular}{@{}cccccccc@{}}
    \toprule
    Waypoint Order      & 1        & 2         & 3         & 4   & 5         & 6          & 7         \\ \midrule
    x [m] & $\gamma$ & 2$\gamma$ & $\gamma$  & 0.0 & -$\gamma$ & -2$\gamma$ & -$\gamma$ \\
    y [m] & $\gamma$ & 0.0       & -$\gamma$ & 0.0 & $\gamma$  & 0.0        & -$\gamma$ \\ \bottomrule
\end{tabular}

\end{center}
\end{table}

\begin{figure*}[!htbp]
\centering
\includegraphics[width=0.48\textwidth]{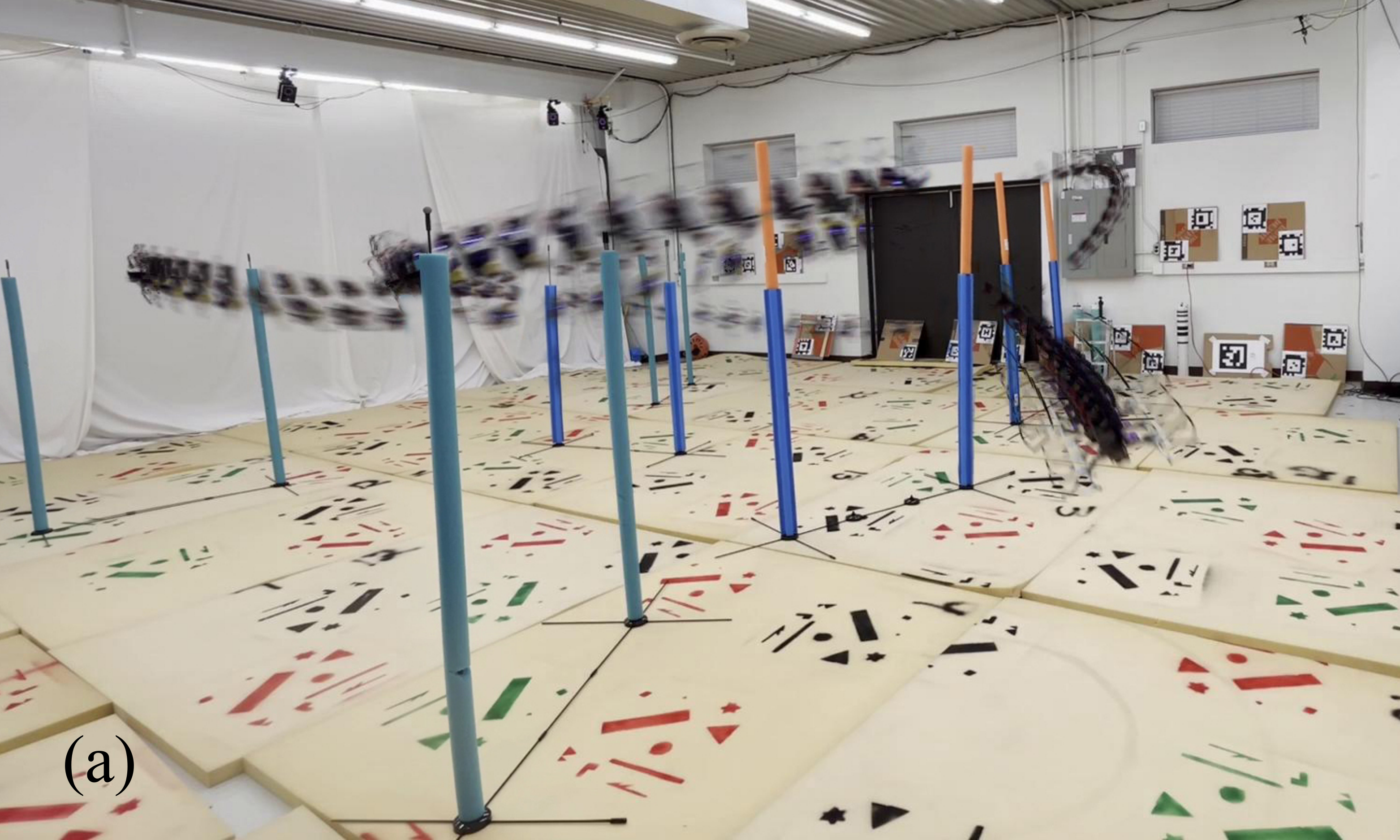}
\includegraphics[width=0.48\textwidth]{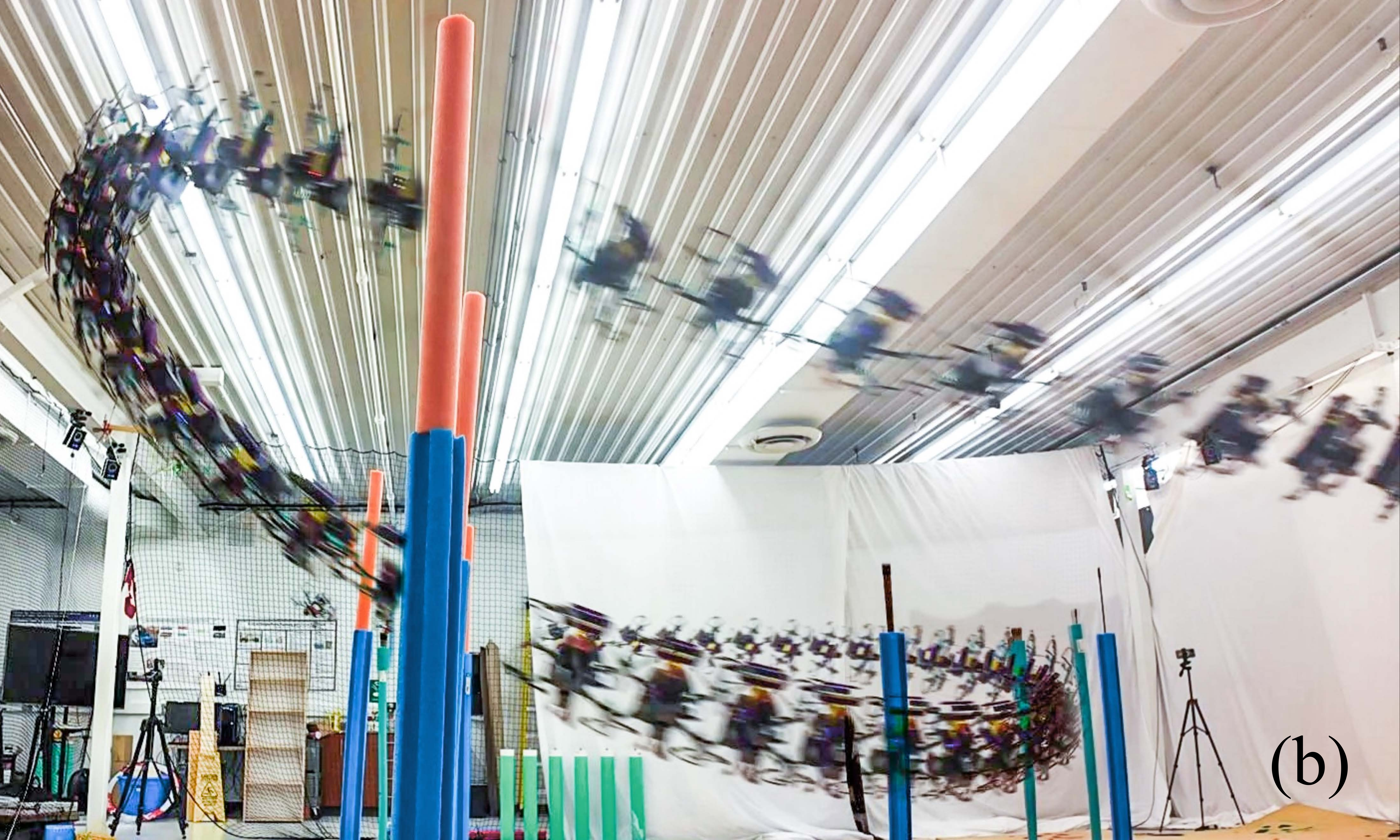}
\includegraphics[trim={25mm 15mm 30mm 25mm},clip,width=1.0\textwidth]{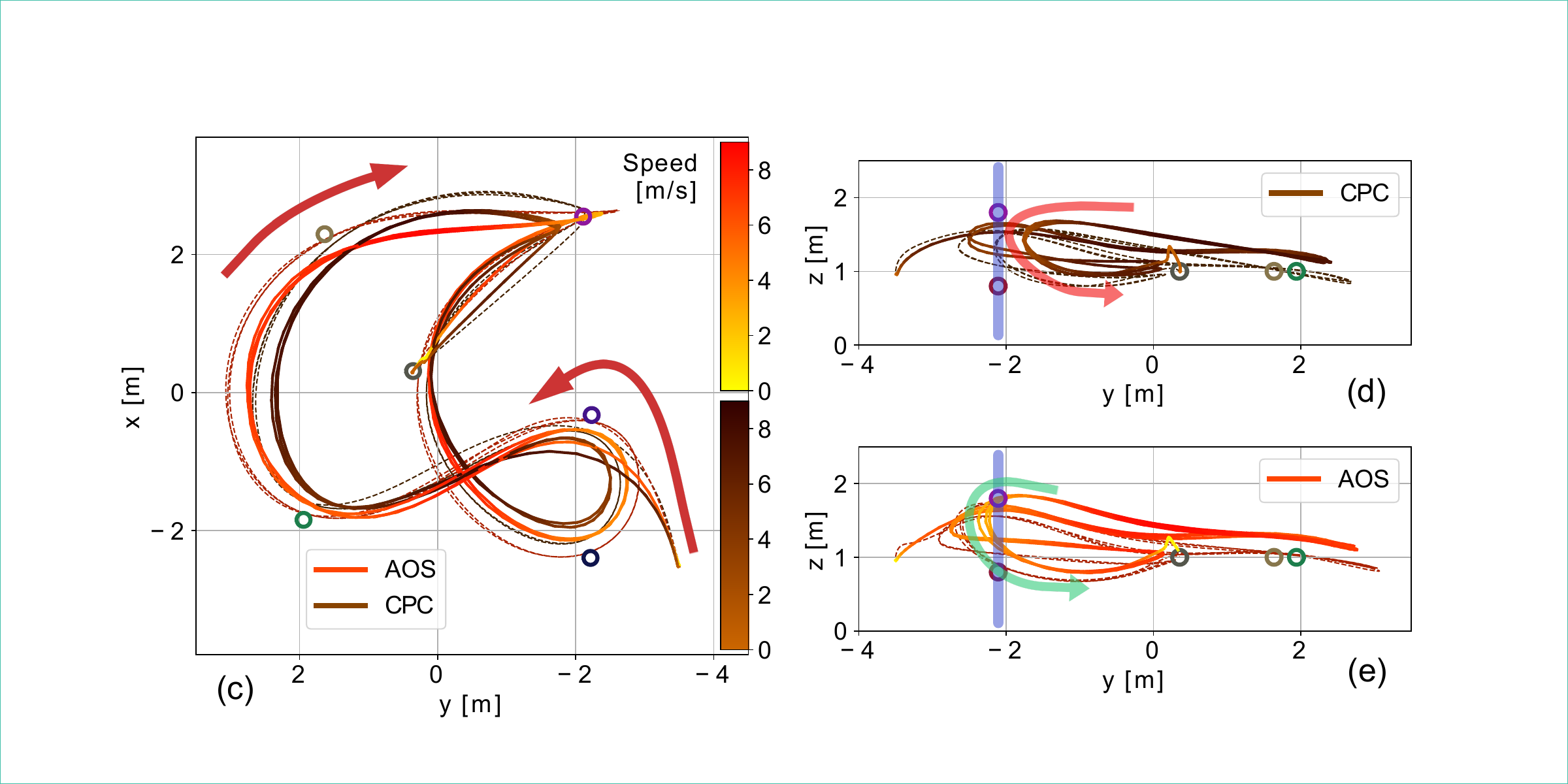}
\vspace{-5mm}
\caption{Comparison of time-optimal trajectories generated by AOS and CPC for the adapted Split-S track with the number of waypoints but slightly different locations. Each waypoint is marked as a circle with a radius of 0.1 m. (a-b) Photos of flight trajectories by using the algorithm presented herein. (c) Planned and actual trajectories of AOS (grey dashed line; red solid line) and CPC (black dashed line; brown solid line). (d-e) Side views when passing the two vertical waypoints. Note that AOS successfully passed through these two points whereas CPC did not.}
\end{figure*}

\subsection{Experimental Validation}

We conduct real-world experiments to validate the proposed algorithm qualitatively. We adapt the Split-S track to fit within the accessible tracking volume and reduce the waypoint tolerance to 0.1 m. The objective is to analyze the performance difference between CPC and AOS.

Fig. \ref{fig_uzh_19wp} presents the executed trajectories from both methods. The result suggests that our approach can successfully navigate the vehicle through all waypoints in 11.0 s, hitting a maximum speed of 8.86 m/s. Although CPC achieves a faster lap time of 10.1 s, we find that it cuts corners seriously and does not pass through the two vertical waypoints shown in the upper right corner. And notably, AOS exhibits better convergence as the planned trajectory strictly passes through each waypoint.

\section{Discussion}

\subsection{Limitations}

Several issues specific to time-optimal trajectory planning may arise when using AOS. First of all, it is unstable to generate flip maneuvers in purely vertical translations even if the singularity has been avoided by strategies introduced in the supplementary material. One possible solution is to introduce the Hopf Fibration Control Algorithm (HFCA) proposed in \cite{watterson2019control} to switch between two geometric constructions of $SO(3)$ in light of the desired thrust direction of the quadrotor. Secondly, AOS may exhibit degraded performance when the waypoints are too close together, especially with a non-collinear layout. This is because, in such circumstances, a primary moving direction between two waypoints may not exist, rendering a two-dimensional model inappropriate for capturing the motion characteristics. But for long-range missions, this situation is quite rare.

\subsection{Generalization of AOS to Other Systems}

The central idea of AOS can be generalized to other systems, albeit not in the exact same manner, and we have demonstrated a successful application to a double integrator in previous sections. The key is to identify a suitable functional subspace that can easily scale with time and has the ability to construct or approximate each trajectory segment appearing in the time-optimal trajectory. Interestingly, this concept can be linked to the theory of functional connection \cite{johnston2021theory} with a context of time-optimal control problems.





\section{Conclusion}


In this paper, we present a polynomial-based quadrotor motion planner that can generate time-optimal trajectories for a wide range of scenarios, e.g., drone racing, and provide necessary theoretical support regarding time optimality. It enables long-rang time-optimal trajectory planning with high time efficiency and is suitable for large-scale operations. Extensive simulations and experiments demonstrate its superior scalability w.r.t. the number of waypoints and mission ranges. Specifically, it maintains an optimality gap within 10\% in the majority of tasks, while consuming significantly shorter computation time than state-of-the-art methods. In future work, we aim to reduce the tracking error by incorporating data-driven quadrotor models in order to enhance trajectory feasibility.

\section*{Acknowledgments}
The authors would like to acknowledge the sponsorship by the Natural Sciences and Engineering Research Council of Canada (NSERC) under grant RGPIN-2023-05148. Additionally, we express our gratitude for the experimental support provided by the Learning Systems \& Robotics Lab (formerly Dynamic Systems Lab) at TUM and UTIAS.


{\appendices
\section*{Appendix A: Proof of Corollary \ref{cor:ut_switch}}

\begin{proof}
``$\Rightarrow$'': Since $\varTheta$ is flat, we have  $c_{2}c_{3}\!=\!c_{1}c_{4}$ as per Lemma \ref{lem:vartheta} $\Rightarrow$ $p_{2}$ and $p_{4}$ can be both zeros by Lemma \ref{lem:phi_p2p4}. Therefore, according to Lemma \ref{lem:phit_0}, there are two possible cases for $\Phi_{T}\!=\!0$ to happen, i.e., either $\theta^{*}(\hat{\tau})=\varTheta(\hat{\tau})\pm\frac{\pi}{2}$ or $p_{2}(\hat{\tau})=p_{4}(\hat{\tau})=0$ hold. 

``$\Leftarrow$'': As established by Lemma \ref{lem:phi_p2p4} and Lemma \ref{lem:phit_0}, either condition (i) or (ii) ensures that $\Phi_{T}(\hat{\tau})\!=\!0$, thereby meeting the necessary condition for a switching time. However, it is unclear whether they do result in a switch because $\Phi_{T}$ may not change sign after $\hat{\tau}$. We need to discuss both cases separately. For condition (i), it is trivial to obtain that $\Phi_{T}$ will flip its sign because $p_{2}$ and $p_{4}$ are both affine functions of time (and thus their signs must change after passing the zeros). For condition (ii), we need to prove that $\hat{\tau}$ is not a local extremal of $\Phi_{T}$. The proof can be done via a contradiction. Assume that condition (ii) holds and $\hat{\tau}$ is not a switching time. This implies that $\Phi_{T}$ keeps the same sign after meeting its zero. Without loss of generality, we assume that $\Phi_{T}\geq0$ within a neighborhood of $\hat{\tau}$ (and therefore $u_{R}^{*}\!=\!1$ within the same interval). It follows that $\hat{\tau}$ is the local minimizer corresponding to $\Phi_{T}\!=\!0$. As a result, we have $\dot{\Phi}_{T}(\hat{\tau})=0$, which can be expanded to:
\begin{equation}
(p_{2}(\hat{\tau})\!-\!c_{3})\cos\theta^{*}(\hat{\tau})\!+\!(-p_{4}(\hat{\tau})\!-\!c_{1})\sin\theta^{*}(\hat{\tau})\!=\!0.
\end{equation}
To meet $\Phi_{T}(\hat{\tau})\!=\!\dot{\Phi}_{T}(\hat{\tau})\!=\!0$, $\theta^{*}$ must satisfy:
\begin{equation}
\theta^{*}(\hat{\tau})\!=\!\arctan\left(\frac{p_{2}(\hat{\tau})-c_{3}}{p_{4}(\hat{\tau})+c_{1}}\right)\!=\!\arctan\left(\frac{-p_{4}(\hat{\tau})}{p_{2}(\hat{\tau})}\right).
\end{equation}
Taking tangents on both sides, multiplying out the constraint, and rearranging the equation yield:
\begin{equation}
p_{2}^{2}(\hat{\tau})+p_{4}^{2}(\hat{\tau})=c_{2}c_{3}-c_{1}c_{4}.
\end{equation}
Since $c_{2}c_{3}=c_{1}c_{4}$ is given, both sides of the equation must be zeros, including $p_{2}(\hat{\tau})=p_{4}(\hat{\tau})=0$. However, this leads to a contradiction to the given condition.
\end{proof}

{\appendices
	\section*{Appendix B: Proof of Lemma \ref{lem:ur_notbb}}

\begin{proof}
We will prove it by making a contradiction. Assume that there exists a phase that $u_{R}^{*}$ is bang-bang with a switching time of $\hat{\tau}$ (and therefore $\Phi_{R}(\hat{\tau})\!=\!0$). We use $\theta_{\hat{\tau}}$ to denote the optimal rotation at $\hat{\tau}$. Regarding the value of $u_{T}^{*}$ within a small enough neighborhood of $\hat{\tau}$, there are two possibilities: (i) $u_{T}^{*}$ does not change value around $\hat{\tau}$ and (ii) $u_{R}^{*}$ and $u_{T}^{*}$ change values at $\hat{\tau}$ simultaneously.

Case (i): Since $u_{T}^{*}$ is unchanged around $\hat{\tau}$, there must exist a $\varepsilon>0$ such that $u_{T}^{*}$ is bang during $[\hat{\tau}-\varepsilon, \hat{\tau}+\varepsilon]$. We introduce an auxiliary function $g:\mathbb{R}\!\rightarrow\!\mathbb{R}$ whose input is a single real value $\delta\in[0,\varepsilon]$:
\begin{equation}
g(\delta)=H^{*}(\hat{\tau}\!+\!\delta)-H^{*}(\hat{\tau}\!-\!\delta).\label{equ_auxil_func}
\end{equation}
We know that $g(\delta)\!=\!0$ for every $\delta\in[0,\varepsilon]$ since the Hamiltonian is 0 along the optimal trajectory. This means that $g(\delta)$ must vanish along the optimal trajectory over $[\hat{\tau}-\varepsilon, \hat{\tau}+\varepsilon]$. If we can show that $g(\delta)$ is a nontrivial analytical function of $\delta$, we establish a contradiction because the zeros of analytical functions are isolated. Now we will delve into $g(\delta)$.


We are aware that there are two possible types of $u_{R}^{*}$ around $\hat{\tau}$, one from $1$ to $-1$, and the other from $-1$ to $1$. Since both cases are identical in terms of trajectory structure, it will suffice to study the former case. Inserting Eq. (\ref{equ_hamiltonian_details}) into Eq. (\ref{equ_auxil_func}) and going through a series of elimination, we arrive at:
\begin{align}
\begin{split}
g(\delta)=0=& 2\delta u_{T}^{*}(c_{1}(\sin(\theta_{\hat{\tau}}-\delta)-\sin\theta_{\hat{\tau}})\\
&+c_{3}(\cos(\theta_{\hat{\tau}}-\delta)-\cos\theta_{\hat{\tau}})).\label{equ_auxil_func_ret}
\end{split}
\end{align}
Note that in the above derivation, we need to rely on the Lebesgue integral because $u_{R}^{*}$ is undefined at $\hat{\tau}$ (and consequently $\ddot{\hat{x}}(\hat{\tau})$, $\ddot{\hat{z}}(\hat{\tau})$ and $\dot{\Phi}_{R}(\hat{\tau})$). Using $2\delta u_{T}^{*}\!\neq\!0$, we get:
\begin{equation}
c_{1}(\sin(\theta_{\hat{\tau}}\!-\!\delta)\!-\!\sin\theta_{\hat{\tau}})\!+\!c_{3}(\cos(\theta_{\hat{\tau}}\!-\!\delta)\!-\!\cos\theta_{\hat{\tau}})\equiv 0.\label{equ_urbangbang_const}
\end{equation}
It is not difficult to find that Eq. (\ref{equ_urbangbang_const}) cannot hold for every $\delta\in[0,\varepsilon]$ unless $c_{1}=c_{3}=0$, which reaches a contradiction. 

Case (ii): Following similar steps, we deduce $g(\delta)$ with $u_{T}^{*}$ following a bang-bang policy and switching value exactly at $\hat{\tau}$. We skip the details as the process is similar. The result shows that $g(\delta)$ is a nontrivial analytical function of $\delta$ given $(c_{1},c_{3})\neq(0,0)$, and hence, $u_{R}^{*}$ is never a bang-bang trajectory.
\end{proof}

{\appendices
	\section*{Appendix C: Proof of Lemma \ref{lem:ur_sbs}}

\begin{proof}
(i): From Lemma \ref{lem:ur_notbb}, we know that $u_{R}^{*}$ is not bang-bang. Moreover, $|\dot{\varTheta}|<1$ almost everywhere is a given condition for every singular flow. As a consequence,  once $\theta^{*}$ leaves a singular flow, it must increase or decrease until it meets another singular flow. In other words, these two singular arcs are on separated singular flows. 

(ii): Being adjacent means that their corresponding singular flows have a gap of $\pi$, and hence, there is no singular flow in between. Since $c,d$ are both switching times for $u_{R}^{*}$, we first have $\Phi_{R}(c)\!=\!\Phi_{R}(d)\!=\!0$. As per the Mean Value Theorem, there must exist a time $e\in(c,d)$ such that $\dot{\Phi}_{R}(e)=0$. To meet this condition, by Corollary \ref{cor:phir_0}, either $p_{2}(e)\!=\!p_{4}(e)\!=\!0$ or $\theta^{*}(e)=\varTheta(e)$ must hold. Under this situation, only the former case is possible because $\theta^{*}$ cannot intersect any singular flow for any time between $(c,d)$.

(iii): Being 2-adjacent means that their corresponding singular flows have a gap of $2\pi$. Using $c_{2}c_{3}\!=\!c_{1}c_{4}$ with $(c_{1},c_{3})\!\neq\!(0,0)$, we know from Lemma \ref{lem:vartheta} that it is a flat singular flow and its value is $\theta^{*}(c) = \arctan(c_{1}/c_{3}) \pm k\pi$ for some $k\in\mathbb{N}_{0}$. We also know that the duration for the bang arc is exactly $2\pi$, i.e., $d\!-\!c\!=\!2\pi$.

Without loss of generality, we assume that the first singular arc lies on $\varTheta$ while the second one is on $\varTheta+2\pi$. It follows that $u_{R}^{*}=1$ on $(c,d)$. For convenience, we set $u_{T}^{*}(c)=\overline{u_{T}}$ as its value doesn't affect the proof, and the same proof can be also carried out with $\underline{u_{T}}$. 

Now, let's prove this statement by contradiction. Assume that $(p_{2}(\hat{t}),p_{4}(\hat{t}))\neq(0,0)$ for all $\hat{t}\in(c,d)$. In this case, Corollary \ref{cor:ut_switch} indicates that there are exactly two switching times for $u_{T}^{*}$ at $c+\pi/2$ and $c+3\pi/2$. This suffices to deduce the optimal thrust control over $(c,d)$: 
\begin{align}
u_{T}^{*}=\begin{cases}
\begin{array}{c}
\overline{u_{T}}\\
\underline{u_{T}}\\
\overline{u_{T}}
\end{array} & \begin{array}{c}
c\leq\hat{t}<c+\pi/2\\
c+\pi/2<\hat{t}<c+3\pi/2\\
c+3\pi/2<\hat{t}<c+2\pi
\end{array}\end{cases}.
\end{align}
By the maximization condition, we have $H^{*}(c)\!=\!H^{*}(d)\!=\!0$, and subsequently, $H^{*}(d)\!-\!H^{*}(c)=0$. Applying Eq. (\ref{equ_hamiltonian_details}), we are able to expand and simply its expression to the following form:
\begin{equation}
(c_{1}\sin\theta^{*}(c)\!+\!c_{3}\cos\theta^{*}(c))((1-\pi)\overline{u_{T}}-\underline{u_{T}})\!=\!0.\label{equ_hdhc}
\end{equation}
Note the caveat that $u_{T}^{*}$ is undefined at the switching time, and hence, the Lebesgue integral will be used to compute the integral. Firstly, we find that $(1-\pi)\overline{u_{T}}\!-\!\underline{u_{T}}\neq0$ by using $\overline{u_{T}},\underline{u_{T}}\!>\!0$ and $1\!-\!\pi<0$. It implies that $c_{1}\sin\theta^{*}(c)+c_{3}\cos\theta^{*}(c)$ must be zero to make the right-hand side of Eq. (\ref{equ_hdhc}) zero. However, this is impossible because $\theta^{*}(c) = \arctan(c_{1}/c_{3}) \pm k\pi$.
\end{proof}

{\appendices
	\section*{Appendix D: Proof of Lemma \ref{lem:nonhorizontal}}

 
\begin{proof}
(i): $c_{2}c_{3}\!=\!c_{1}c_{4}$ implies that singular flows are flat, i.e., $\varTheta$ is a constant $\Rightarrow$ $\dot{\varTheta}=0$. Moreover, 
Lemma \ref{lem:ur_singular} indicates that when $u_{R}^{*}$ is singular, we must have $\theta^{*}=\varTheta$, and therefore $u_{R}^{*}=0$.

(ii): We have already understood that $u_{R}^{*}$ does not have two consecutive bang arcs from Lemma \ref{lem:ur_notbb}. The key to proving (ii) is to show that there is an upper bound on the number of singular arcs for $u_{R}^{*}$, and this upper bound is two. Assume that $u_{R}^{*}$ contains three singular arcs. According to Lemma \ref{lem:ur_sbs}, there must exist two places where $p_{2}\!=\!p_{4}\!=\!0$. However, this is impossible as per  Lemma \ref{lem:p2p4}. Therefore, $u_{R}^{*}$ has at most two singular arcs. Subsequently, it can easily deduce the most complicated pattern for $u_{R}^{*}$ is B-S-B-S-B, which consists of 4 switches. For the maximum switch number of $u_{T}^{*}$, one just needs to calculate the possible intersections between $\theta^{*}$ and the singular flows. Taking into account an additional switch happening when $p_{2}\!=\!p_{4}\!=\!0$ (by Lemma \ref{lem:phit_0}), we know that the most complex $u_{T}^{*}$ is B-B with at most 5 switches.
\end{proof}
 
{\appendices
	\section*{Appendix E: Proof of Lemma \ref{lem:horizontal}}

\begin{proof}
(i): See Lemma \ref{lem:phi_p2p4}.

(ii): Lemma \ref{lem:prior_works} gives the expression of $u_{R}^{*}$ when singular. Moreover, it is nonzero because both its denominator and numerator are nonzero. This can be verified by demonstrating that the root of the denominator does not exist with $\mathbf{c}\!\neq\!\mathbf{0}$.

(iii): Because of $(p_{2},p_{4}\neq(0,0)$, it is impossible to have two singular arcs for $u_{R}^{*}$ as per Lemma \ref{lem:ur_sbs}. Furthermore, the given condition has suggested that the only singular arc must lie on the central singular flow. Therefore, the most complicated $u_{R}^{*}$ allowed is B-S-B. In this case, $\theta^{*}$ can at most pass $\varTheta\pm \frac{\pi}{2}$ twice and thus $u_{T}^{*}$ contains at most 2 switches.
\end{proof}

\bibliographystyle{IEEEtran}
\bibliography{IEEEabrv,aos}

\end{document}